\documentclass{article}
\usepackage[margin=1.5in]{geometry}

\usepackage{times}
\usepackage{graphicx}
\usepackage{subfigure} 
\usepackage{natbib}
\usepackage{algorithm}
\usepackage{algorithmic}
\usepackage{amsmath}
\usepackage{amssymb}
\usepackage{mathtools}
\usepackage{multirow}
\usepackage{hyperref}
\usepackage{color}

\newtheorem{thm}{Theorem}
\newtheorem{dfn}{Definition}

\newtheorem{lem}{Lemma}

\newenvironment{proof}{\noindent{\bf {Proof: }}\ }{\hfill$\blacksquare$ \vspace{1mm}}

\newtheorem{claim}{Claim}

\newcommand{\fmk}{\mathbf{F}^k_m}
\newcommand{\fmt}{\mathbf{F}^2_4}
\newcommand{\hmk}{\mathbf{H}^k}
\newcommand{\rank}{\text{rank}}
\newcommand{\krank}{\text{rank}_K}

\newcommand\Omit[1]{}

\title{Learning Mixtures of Plackett-Luce Models}

\author{Zhibing Zhao\\
Computer Science Department\\
Rensselaer Polytechnic Institute\\
Troy, NY 12180 USA\\
\texttt{zhaoz6@rpi.edu}
\and 
Peter Piech\\
Computer Science Department\\
Rensselaer Polytechnic Institute\\
Troy, NY 12180 USA\\
\texttt{piechp@rpi.edu}
\and 
Lirong Xia\\
Computer Science Department\\
Rensselaer Polytechnic Institute\\
Troy, NY 12180 USA\\
\texttt{xial@cs.rpi.edu}
}

\date{}

\begin{document}

\maketitle

\begin{abstract}%
In this paper we address the identifiability and efficient learning problems of finite mixtures of Plackett-Luce models for rank data. We prove that for any $k\geq 2$, the mixture of $k$ Plackett-Luce models for no more than $2k-1$ alternatives is non-identifiable and this bound is tight for $k=2$.
For generic identifiability, we prove that the mixture of $k$ Plackett-Luce models over $m$ alternatives is {\em generically identifiable} if $k\leq\lfloor\frac {m-2} 2\rfloor!$. 
We also propose an efficient generalized method of moments (GMM) algorithm to learn the mixture of two Plackett-Luce models and show that the algorithm is consistent. Our experiments show that our GMM algorithm is significantly faster than the EMM algorithm by \citet{Gormley08:Exploring}, while achieving competitive statistical efficiency.
\end{abstract}

\section{Introduction}
\label{intro}

In many machine learning problems the data are composed of rankings over a finite number of {\em alternatives}~\cite{Marden95:Analyzing}. For example, meta-search engines aggregate rankings over webpages from individual search engines~\cite{Dwork01:Rank}; rankings over documents are combined to find the most relevant document in information retrieval~\cite{Liu11:Learning}; noisy answers from online workers are aggregated to produce a more accurate answer in crowdsourcing~\cite{Mao13:Better}. Rank data are also very common in economics and political science. For example, consumers often give discrete choices data~\cite{McFadden74:Conditional} and voters often give rankings over presidential candidates~\cite{Gormley08:Exploring}.

Perhaps the most commonly-used statistical model for  rank data is the {\em Plackett-Luce} model~\cite{Plackett75:Analysis,Luce59:Individual}. The Plackett-Luce model is a natural generalization of multinomial logistic regression. In a Plackett-Luce model, each alternative is parameterized by a positive number that represents the ``quality'' of the alternative. The greater the quality, the the chance the alternative will be ranked at a higher position.

In practice, {\em mixtures} of Plackett-Luce models can provide better fitness than a single Plackett-Luce model. An additional benefit is that the learned parameter of a mixture model can naturally be used for clustering~\cite{McLachlan88:Mixture}. The $k$-mixture of Plackett-Luce combines $k$ individual Plackett-Luce models via a linear vector of {\em mixing coefficients}. 
For example, \citet{Gormley08:Exploring} propose an {\em Expectation Minorization Maximization (EMM)} algorithm to compute the MLE of Plackett-Luce mixture models. The EMM was applied  to an Irish election dataset with $5$ alternatives and the four components in the mixture model are interpreted as {\em voting blocs}.

Surprisingly, the {\em identifiability} of Plackett-Luce mixture models is still unknown. Identifiability is an important property for statistical models, which requires that different parameters of the model have different distributions over samples. Identifiability is crucial because if the model is not identifiable, then there are cases where it is impossible to estimate the parameter from the data, and in such cases conclusions drawn from the learned parameter can be wrong. 
In particular, if Plackett-Luce mixture models are not identifiable, then the voting bloc produced by the EMM algorithm of \citet{Gormley08:Exploring} can be dramatically different from the ground truth.

In this paper, we address the following two important questions about the theory and practice of Plackett-Luce mixture models for rank data.

\noindent{\bf Q1.} Are Plackett-Luce mixture models identifiable? 

\noindent{\bf Q2.} How can we efficiently learn Plackett-Luce mixture models? 

Q1 can be more complicated than one may think because the non-identifiability of a mixture model usually comes from two sources. The first is {\em label switching}, which means that if we label the components of a mixture model differently, the distribution over samples does not change~\cite{Stephens00:Dealing}. This can be avoided by ordering the components and merging the same components in the mixture model. The second is more fundamental, which states that the mixture model is non-identifiable even after ordering and merging duplicate components. Q1 is about the second type of non-identifiability. 

The EMM algorithm by~\citet{Gormley08:Exploring} converges to the MLE, but as we will see in the experiments, it can be very slow when the data size is not too small. Therefore, to answer Q2, we want to design learning algorithms that are much faster than the EMM without sacrificing too much statistical efficiency, especially mean squared error (MSE).

\subsection{Our Contributions}\label{sec:cont}
We answer Q1 with the following theorems. The answer depends on the number of components $k$ in the mixture model and the number of alternatives $m$. 

\vspace{2mm}
{\noindent\bf Theorem~\ref{thm:nid} and~\ref{thm:id}.} {\em 
For any $m\ge 2$ and any $k\geq \frac{m+1}{2}$, the $k$-mixture Plackett-Luce model (denoted by $k$-PL) is {\bf non-identifiable}. This lower bound on $k$ as a function of $m$ is tight for $k=2$ ($m=4$). }
\vspace{1mm}

The second half of the theorem is positive: the mixture of two Plackett-Luce models is identifiable for four or more alternatives.  We conjecture that the bound is tight for all $k>2$.

The $k$-PL is {\em generically} identifiable for $m$ alternatives, if the Lebesgue measure of non-identifiable parameters is $0$. We prove the following positive results for $k$-PL.

\vspace{2mm}
{\noindent\bf Theorem~\ref{thm:gid}.} {\em 
For any $m\ge 6$ and any $k\le  \lfloor\frac {m-2} 2\rfloor!$, the $k$-PL is generically identifiable. }
\vspace{2mm}

We note that $\lfloor\frac {m-2} 2\rfloor!$ is exponentially larger than the lower bound $\frac{m+1}{2}$ for (strict) identifiability. One interpretation of the theorem is that, when $\frac{m}{2}+1\le k\le \lfloor\frac {m-2} 2\rfloor!$, even though $k$-PL is not identifiable in the strict sense, one may not need to worry too much in practice due to generic identifiability.

For Q2, we propose a generalized method of moments (GMM)\footnote{This should not be confused with {\em Gaussian mixture models}.} algorithm~\citep{Hansen82:Large} to learn the $k$-PL. We illustrate the algorithm for $k=2$ and $m\ge 4$, and prove that the algorithm is consistent, which means that when the data are generated from $k$-PL and the data size $n$ goes to infinity, the algorithm will reveal the ground truth with probability that goes to $1$. We then compare our GMM algorithm and the EMM algorithm~\cite{Gormley08:Exploring} w.r.t.~statistical efficiency (mean squared error) and computational efficiency in synthetic experiments. As we will see, in Section~\ref{sec:exp}, our GMM algorithm is significantly faster than the EMM algorithm while achieving competitive statistical efficiency. Therefore, we believe that our GMM algorithm is a promising candidate for learning Plackett-Luce mixture models from big rank data.

\subsection{Related Work and Discussions}
Most previous work in mixture models (especially Gaussian mixture models) focuses on cardinal data~\cite{Teicher61:Identifiability,Teicher63:Identifiability,McLachlan04:Finite,Kalai12:Disentangling, Dasgupta99:Learning}. Little is known about the identifiability of mixtures of models for rank data. 

For rank data, \citet{Iannario10:On-the-identifiability} proved the identifiability of the mixture of shifted binomial model and the uniform models.  \citet{Awasthi14:Learning} proved the identifiability of mixtures of two Mallows' models. Mallows mixture models were also studied by \citet{Lu14:Effective} and \citet{Chierichetti15:On}. Our paper, on the other hand, focuses on mixtures of Plackett-Luce models without any assumptions on data since identifiability is a property of a statistical model. \citet{Ammar14:What} discussed non-identifiability of Plackett-Luce mixture models when data consist of {\em only} limited length of partial rankings. \citet{Oh14:Learning} provided sufficient conditions on data for Plackett-Luce mixture models to be identifiable. \citet{Suh16:Adversarial} studied the possibility to identify the top $K$ alternatives of two mixtures of Plackett-Luce models with respect to the size of datasets. 
  
Technically, part of our (non-)identifiability proofs is motivated by the work of~\citet{Teicher63:Identifiability}, who obtained sufficient conditions for the identifiability of finite mixture models. However, technically these conditions cannot be directly applied to $k$-PL because they work either for finite families (Theorem~1 in~\cite{Teicher63:Identifiability}) or for cardinal data (Theorem~2 in~\cite{Teicher63:Identifiability}). Neither is the case for mixtures of Plackett-Luce models. To prove our (non-)identifiability theorems, we develop novel applications of the Fundamental Theorem of Algebra to analyze the rank of  a matrix $\fmk$ that represents $k$-PL (see Preliminaries for more details). Our proof for generic identifiability is based on a novel application of the tensor-decomposition approach that analyzes the generic {\em Kruskal's rank} of matrices  advocated by~\citet{Allman09:Identifiability}.

In addition to being important in their own right, our (non)-identifiability theorems also carry a clear message that has been overlooked in the literature: when using Plackett-Luce mixture models to fit rank data, one must be very careful about the interpretation of the learned parameter. Specifically, when $m\le 2k-1$, it is necessary to double-check whether the learned parameter is identifiable (Theorem~\ref{thm:nid}), which can be computationally hard. On the positive side, identifiability may not be a big concern in practice under a much milder condition ($k\le \lfloor\frac {m-2} 2\rfloor!$, Theorem~\ref{thm:gid}). 

\citet{Gormley08:Exploring} used $4$-PL to fit an Irish election dataset with $5$ alternatives. According to our Theorem~\ref{thm:nid}, $4$-PL for $5$ alternatives  is non-identifiable. Moreover, our generic identifiability theorem (Theorem~\ref{thm:gid}) does not apply because $m=5<6$. Therefore, it is possible that there exists another set of voting blocs and mixing coefficients with the same likelihood as the output of the EMM algorithm. Whether it is true or not, we believe that it is important to add discussions and justifications of the uniqueness of the voting blocs obtained by~\citet{Gormley08:Exploring}.

Parameter inference for single Plackett-Luce models is studied in \cite{Cheng10:Label} and \cite{Azari13:Generalized}. \citet{Azari13:Generalized} proposed a GMM, 
which is quite different from our method, and cannot be applied to Plackett-Luce mixture models. The MM algorithm by \citet{Hunter04:MM}, which is compared in \cite{Azari13:Generalized}, is also very different from the EMM that is being compared in this paper.

\section{Preliminaries} 
 
Let $\mathcal{A}=\{a_i|i=1, 2, \cdots, m\}$ denote a set of $m$ alternatives. Let $\mathcal{L}(\mathcal{A})$ denote the set of linear orders (rankings), which are transitive, antisymmetric and total binary relations, over $\mathcal{A}$. A ranking is often denoted by $a_{i_1}\succ a_{i_2}\succ\cdots\succ a_{i_m}$, which means that $a_{i_1}$ is the most preferred alternative, $a_{i_2}$ is the second preferred, $a_{i_m}$ is the least preferred, etc. Let $P=(V_1, V_2, \cdots, V_n)$ denote the data (also called a {\em preference profile}), where for all $j\leq n, V_j\in\mathcal{L}(\mathcal{A})$.

\begin{dfn}
(Plackett-Luce model). The parameter space is $\Theta=\{\vec{\theta}=\{\theta_i|i=1, 2, \cdots, m, \theta_i\in[0, 1], \sum_{i=1}^m\theta_i=1\}\}$. The sample space is $\mathcal{S}=\mathcal{L}(\mathcal{A})^n$. Given a parameter  $\vec\theta\in\Theta$, the probability of any ranking $V=a_{i_1}\succ a_{i_2}\succ\cdots\succ a_{i_m}$ is

\hfill$\Pr\nolimits_\text{PL}(V|\vec{\theta})=\frac {\theta_{i_1}} {1}\times\frac {\theta_{i_2}}
{\sum_{p>1}\theta_{i_p}}\times\cdots\times\frac {\theta_{i_{m-1}}} {\theta_{i_{m-1}}+\theta_{i_m}}\hfill$

\end{dfn}

We assume that data are generated i.i.d.~in the Plackett-Luce model. Therefore, given a preference profile $P$ and $\vec \theta\in\Theta$, we have
$
\Pr_\text{PL}(P|\vec{\theta})=\prod^n_{j=1}\Pr_\text{PL}(V_j|\vec{\theta})
$.

The  Plackett-Luce model has the following intuitive explanation. Suppose there are $m$ balls, representing $m$ alternatives in
an opaque bag. Each ball $a_i$ is assigned a quality value $\theta_i$. Then, we generate a ranking in $m$ stages. In each stage, we take one ball out of the bag. The probability for each remaining ball being taken out is the value assigned to it over the sum of the values assigned to the remaining balls. The order of drawing is the ranking over the alternatives.

We require $\sum_i\theta_i=1$ to normalize the parameter so that the Plackett-Luce model is identifiable. It is not hard to verify that for any Plackett-Luce model, the probability for the alternative $a_p$ ($p\leq m$)  to be ranked at the top of a ranking is $\theta_p$; the probability for $a_p$ to be ranked at the top and $a_q$ ranked at the second position is $\frac {\theta_p\theta_q} {1-\theta_p}$, etc.

\begin{dfn}
($k$-mixture Plackett-Luce model). Given $m\geq 2$ and $k\geq 2$, we define the {\em $k$-mixture Plackett-Luce model} as follows. The sample space is $\mathcal{S}=\mathcal{L}(\mathcal{A})^n$. The parameter space has two parts. The first part is the {\em mixing coefficients} $(\alpha_1,\ldots, \alpha_k)$ where for all $r\leq k$, $\alpha_r\ge0$, and $\sum_{r=1}^k\alpha_r=1$. The second part is $(\vec \theta^{(1)},\vec \theta^{(2)},\ldots,\vec \theta^{(k)})$, where $\vec \theta^{(r)}=[\theta^{(r)}_1, \theta^{(r)}_2,\cdots,\theta^{(r)}_m]^\top$ is the parameter of the $r$-th Plackett-Luce component. The probability of a ranking $V$ is

\hfill$
\Pr\nolimits_{k\text{-PL}}(V|\vec{\theta})=\sum^k_{r=1}\alpha_r\Pr\nolimits_{\text{PL}}(V|\vec\theta^{(r)})
,\hfill$

where $\Pr\nolimits_{\text{PL}}(V|\vec\theta^{(r)})$ is the probability of $V$ in the $r$-th Plackett-Luce model given $\vec \theta^{(r)}$.
\end{dfn}
For simplicity we use $k$-PL to denote the $k$-mixture Plackett-Luce model.

\begin{dfn}
(Identifiability) Let $\mathcal{M}=\{\Pr(\cdot|\vec{\theta}): \vec{\theta}\in\Theta\}$ be a statistical model. $\mathcal{M}$ is identifiable if  for all $\vec{\theta}, \vec{\gamma}\in\Theta$, we have $
\Pr(\cdot|\vec{\theta})=\Pr(\cdot|\vec{\gamma})\Longrightarrow
\vec{\theta}=\vec{\gamma}$.
\end{dfn}
In this paper, we slightly modify this definition to eliminate the label switching problem. We say that $k$-PL is identifiable if there do not exist (1) $1\le k_1, k_2\leq k$, {\em non-degenerate} $\vec\theta^{(1)},  \vec\theta^{(2)}, \cdots, \vec\theta^{(k_1)}$, $\vec\gamma^{(1)}, \vec\gamma^{(2)}, \cdots, \vec\gamma^{(k_2)}$, which means that these $k_1+k_2$ vectors are pairwise different; (2) all strictly positive mixing coefficients $(\alpha_1^{(1)},\ldots,\alpha_{k_1}^{(1)})$ and $(\alpha_1^{(2)},\ldots,\alpha_{k_2}^{(2)})$, so that for all rankings $V$ we have

$\hfill\sum^{k_1}_{r=1}\alpha^{(1)}_r\Pr\nolimits_{\text{PL}}(V|\vec\theta^{(r)})=\sum^{k_2}_{r=1}\alpha^{(2)}_r\Pr\nolimits_{\text{PL}}(V|\vec\gamma^{(r)})\hfill$

Throughout the paper, we will represent a distribution over the $m!$ rankings over $m$ alternatives for a Plackett-Luce component with parameter $\vec \theta^{(r)}$ as a column vector $\vec f_m(\vec\theta)$ with $m!$ elements, one for each ranking and whose value is the probability of the corresponding ranking. For example, when $m=3$, we have
\begin{equation*}%
\vec{f_3}(\vec{\theta})=
\begin{pmatrix}
\Pr(a_1\succ a_2\succ a_3|\vec{\theta})\\
\Pr(a_1\succ a_3\succ a_2|\vec{\theta})\\
\Pr(a_2\succ a_1\succ a_3|\vec{\theta})\\
\Pr(a_2\succ a_3\succ a_1|\vec{\theta})\\
\Pr(a_3\succ a_1\succ a_2|\vec{\theta})\\
\Pr(a_3\succ a_2\succ a_1|\vec{\theta})
\end{pmatrix}=
\begin{pmatrix}
\frac {\theta_1\theta_2} {1-\theta_1}\\
\frac {\theta_1\theta_3} {1-\theta_1}\\
\frac {\theta_1\theta_2} {1-\theta_2}\\
\frac {\theta_2\theta_3} {1-\theta_2}\\
\frac {\theta_1\theta_3} {1-\theta_3}\\
\frac {\theta_2\theta_3} {1-\theta_3}
\end{pmatrix}
\end{equation*}

Given $\vec \theta^{(1)},\ldots,\vec{\theta^{(2k)}}$, we define $\mathbf{F}^k_m$ as a $m!\times 2k$ matrix for $k$-PL with $m$ alternatives
\begin{equation}\label{fkm}
\mathbf{F}^k_m=
\begin{bmatrix}
\vec{f_m}(\vec\theta^{(1)}) & \vec{f_m}(\vec\theta^{(2)}) & \cdots & \vec{f_m}(\vec\theta^{(2k)})
\end{bmatrix}
\end{equation}
We note that $\mathbf{F}^k_m$ is a function of $\vec \theta^{(1)},\ldots,\vec\theta^{(2k)}$, which are often omitted. We prove the identifiability or non-identifiability of $k$-PL by analyzing the rank of $\mathbf{F}^k_m$. The reason that we consider $2k$ components is that we want to find (or argue that we cannot find) another $k$-mixture model that has the same distribution as the original one.

\section{Identifiability of Plackett-Luce Mixture Models}

We first prove a general lemma to reveal a relationship between the rank of $\fmk$ and the identifiability of Plackett-Luce mixture models. We recall that a set of vectors is non-degenerate if its elements are pairwise different.

\begin{lem}
\label{lem:rel}
If the rank of $\fmk$ is $2k$ for all non-degenerate $\vec \theta^{(1)},\ldots,\vec\theta^{(2k)}$, then $k$-PL is identifiable. Otherwise $(2k-1)$-PL is non-identifiable.
\end{lem}
\begin{proof}Suppose for the sake of contradiction the rank of $\fmk$ is $2k$ for all non-degenerate $\vec \theta^{(1)},\ldots,\vec\theta^{(2k)}$ but $k$-PL is non-identifiable. Then, there exist non-degenerate $\vec\theta^{(1)},  \vec\theta^{(2)}, \cdots, \vec\theta^{(k_1)}$, $\vec\gamma^{(1)}, \vec\gamma^{(2)}, \cdots, \vec\gamma^{(k_2)}$ and all strictly positive mixing coefficients $(\alpha_1^{(1)},\ldots,\alpha_{k_1}^{(1)})$ and $(\alpha_1^{(2)},\ldots,\alpha_{k_2}^{(2)})$,  such that for all rankings $V$, we have

$\hfill\sum^{k_1}_{r=1}\alpha^{(1)}_r\Pr\nolimits_{\text{PL}}(V|\vec\theta^{(r)})=\sum^{k_2}_{r=1}\alpha^{(2)}_r\Pr\nolimits_{\text{PL}}(V|\vec\gamma^{(r)})\hfill$

Let $\vec\delta^{(1)}, \vec\delta^{(2)},\ldots,\vec\delta^{(2k-(k_1+k_2))}$ denote any $2k-(k_1+k_2)$ vectors so that $\{\vec\theta^{(1)}, \ldots, \vec\theta^{(k_1)}, \vec\gamma^{(1)}, \ldots,\\ \vec\gamma^{(k_2)}, \vec\delta^{(1)}, \ldots, \vec\delta^{(2k-(k_1+k_2))}\}$ is non-degenerate. It follows that the rank of the corresponding $\fmk$ is strictly smaller than $2k$, because
\begin{equation*}
\sum^{k_1}_{r=1}\alpha^{(1)}_r\Pr\nolimits_{\text{PL}}(V|\vec\theta^{(r)})-\sum^{k_2}_{r=1}\alpha^{(2)}_r\Pr\nolimits_{\text{PL}}(V|\vec\gamma^{(r)})+\sum_{r=1}^{(2k-k_1-k_2)}\vec\delta^{(r)}\cdot 0=0.
\end{equation*}
This is a contradiction.

On the other hand, if $\rank(\fmk)<2k$ for some non-degenerate $\vec \theta$'s, then there exists a nonzero vector $\vec\alpha=[\alpha_1, \alpha_2, \ldots, \alpha_{2k}]^\top$ such that $\mathbf{F}^k_m\cdot\vec\alpha=0$.  Suppose in $\vec\alpha$ there are $k_1$ positive elements and $k_2$ negative elements, then it follows that $\max\{k_1,k_2\}$-mixture model is not identifiable, and $\max\{k_1,k_2\}\le 2k-1$.
\end{proof}

\begin{thm}\label{thm:nid}
For any $m\ge 2$ and any $k\geq \frac{m+1}{2}$, the $k$-PL is {non-identifiable}.
\end{thm}
\begin{proof} The proof is constructive and is based on a refinement of the second half of Lemma~\ref{lem:rel}. For any $k$ and $m=2k-1$, we will define $\vec\theta^{(1)},\ldots,\vec\theta^{(2k)}$ and $\vec\alpha=[\alpha_1,\ldots,\alpha_{2k}]^T$ such that (1) $\fmk\cdot \vec\alpha=0$ and (2) $\vec \alpha$ has $k$ positive elements and  $k$ negative elements. In each $\vec\theta^{(r)}$, the value for alternatives $\{a_2,\ldots,a_m\}$ are the same. The proof for any $m<2k-1$ is similar. 

Formally, let $m=2k-1$. For all $i\geq 2$ and $r\leq 2k$, we let $\theta^{(r)}_i=\frac {1-\theta^{(r)}_1} {2k-2}$, where $\theta^{(r)}_i$ is the parameter corresponding to the $i$th alternative of the $r$th model. We use $e_r$ to represent $\theta^{(r)}_1$ 
and we use $b_r$ to represent $\frac {1-\theta^{(r)}_1} {2k-2}$. It is not hard to check that 
the probability for $a_1$ to be ranked at the $i$th position in the $r$th Plackett-Luce model is
\begin{equation}\label{equmain:prob}
\frac{(2k-2)!}{(2k-1-i)!}\frac {e_r(b_r)^{i-1}}
{\prod^{i-1}_{p=0}(1-pb_r)}
\end{equation}
Then $\fmk$ can be reduced to a $(2k-1)\times (2k)$ matrix. Because $\rank(\fmk)\le 2k-1<2k$, Lemma~\ref{lem:rel} immediately tells us that $(2k-1)$-PL is non-identifiable for $2k-1$ alternatives, but this is much weaker than what we are proving in this theorem. We now define a new $(2k-1)\times (2k)$ matrix $\hmk$ obtained from $\fmk$ by performing the following linear operations on row vectors. (i) Make the first row of $\hmk$ to be $\vec 1$; (ii) for any $2\leq i\leq 2k-1$, the $i$th row of $\hmk$ is the probability for $a_1$ to be ranked at the $(i-1)$-th position according to \eqref{equmain:prob}; (iii) remove all constant factors.

More precisely, for any $e_r$ we define the following function.
\begin{equation*}
\vec{f^*}(e_r)=
\begin{pmatrix}
1\\
e_r\\
\frac {e_r(1-e_r)} {e_r+2k-3}\\
\vdots\\
\frac {e_r(1-e_r)^{2k-3}} {(e_r+2k-3)\cdots((2k-3)e_r+1)}
\end{pmatrix}
\end{equation*}

Then we define $
\hmk=[\vec{f^*} (e_1), \vec{f^*}(e_2), \cdots, \vec{f^{*}}(e_{2k})]
$.

\begin{lem}\label{lem:exist}
If there exist all different $e_1, e_2, \cdots, e_{2k}<1$ and a non-zero vector $\vec{\beta^\ast}=[\beta^\ast_1, \beta^\ast_2, \cdots, \beta^\ast_{2k}]^\top$ such that (i) $\hmk\vec{\beta^\ast}=0$ and (ii) $\vec{\beta^\ast}$ has $k$ positive elements and $k$ negative elements, then $k$-PL for $2k-1$ alternatives is not identifiable.
\end{lem}
\begin{proof}
W.l.o.g. assume $\beta^\ast_1, \beta^\ast_2, \cdots, \beta^\ast_k>0$ and $\beta^\ast_{k+1}, \beta^\ast_{k+2},
\beta^\ast_{2k}<0$. $\mathbf{H}^{k}_{2k-1}\vec{\beta^\ast}=0$ means that
\begin{equation*}
\sum^k_{r=1}\beta^\ast_r\vec{f_r}=-\sum^{2k}_{r=k+1}\beta^\ast_r\vec{f_r}
\end{equation*}
According to the first row in $\hmk$, we have $\sum_r\beta^\ast_r=0$. Let $S=\sum^k_{r=1}\beta^\ast_r$. Further let
$\alpha^\ast_r=\beta^\ast_r/S$ when $r=1, 2, \cdots, k$ and $\alpha^\ast_r=-\beta^\ast_r/S$ when $r=k+1, k+2, \cdots, 2k$. We have
\begin{equation*}
\sum^k_{r=1}\alpha^\ast_r\vec{f_r}=\sum^{2k}_{r=k+1}\alpha^\ast_r\vec{f_r}
\end{equation*}
where $\sum^k_{r=1}\alpha^\ast_r=1$ and $\sum^{2k}_{r=k+1}\alpha^\ast_r=1$. This means that  the model is not identifiable.
\end{proof}
Then, the theorem is proved by showing that the following $\vec \beta^*$ satisfies the conditions in Lemma~\ref{lem:exist}. For any $r\le 2k$, 
\begin{equation}\label{equ:mainbeta}
\beta^\ast_r=\frac {\prod^{2k-3}_{p=1} (p e_r+2k-2-p)} {\prod_{q\neq r} (e_r-e_q)}
\end{equation}
Note that the numerator is always positive. 
W.l.o.g. let $e_1<e_2<\cdots<e_{2k}$, then half of the denominators are positive and the other half are negative. We then use induction to prove that the conditions in Lemma~\ref{lem:exist} are satisfied in the following series of lemmas.

\begin{lem}\label{lem:zero}
$\sum_s\frac 1 {\prod_{t\neq s}(e_s-e_t)}=0$. 
\end{lem}

\begin{proof}
The partial fraction decomposition of the first term is
\begin{equation*}
\frac 1 {\prod_{q\neq 1}(e_1-e_q)}=\sum_{q\neq 1}(\frac {B_q} {e_1-e_q})
\end{equation*}
where $B_q=\frac 1 {\prod_{p\neq q, p\neq 1}(e_q-e_p)}$.

Namely,
\begin{equation*}
\frac 1 {\prod_{q\neq 1}(e_1-e_q)}=-\sum_{q\neq 1}(\frac 1
{\prod_{p\neq q}(e_q-e_p)})
\end{equation*}
We have
\begin{equation*}
\sum_s\frac 1 {\prod_{t\neq s}(e_s-e_t)}=\frac 1 {\prod_{q\neq
    1}(e_1-e_q)}+\sum_{q\neq 1}(\frac 1 {\prod_{p\neq q}(e_q-e_p)})=0
\end{equation*}
\end{proof}

\begin{lem}\label{lem:order} For all $\mu\leq\nu-2$, we have
$\sum^\nu_{s=1}\frac {(e_s)^\mu} {\prod_{t\neq s}(e_s-e_t)}=0$. 
\end{lem} 

\begin{proof}
Base case: When $\nu=2, \mu=0$, obviously
\begin{equation*}
\frac 1 {e_1-e_2}+\frac 1 {e_2-e_1}=0
\end{equation*}
Assume the lemma holds for $\nu=p$ and all $\mu\leq\nu-2$, that is $\sum^\nu_{s=1}\frac {e_s^\mu} {\prod_{t\neq s}(e_s-e_t)}=0$. When $\nu=p+1, \mu=0$, by Lemma \ref{lem:zero} we have 
\begin{equation*}
\sum^{p+1}_{s=1}\frac 1 {\prod_{t\neq s}(e_s-e_t)}=0
\end{equation*}
Assume $\sum^{p+1}_{s=1}\frac {e_s^q} {\prod_{t\neq s}(e_s-e_t)}=0$
for all $\mu=q, q\leq p-2$. For $\mu=q+1$,
\begin{align*}
\sum^{p+1}_{s=1}\frac {e_s^{q+1}} {\prod_{t\neq s}(e_s-e_t)}=&
\sum^{p+1}_{s=1}\frac {e_s^q e_{p+1}} {\prod_{t\neq s}(e_s-e_t)}+
\sum^{p+1}_{s=1}\frac {e_s^q (e_s-e_{p+1})} {\prod_{t\neq s}(e_s-e_t)}\\
=&e_{p+1}\sum^{p+1}_{s=1}\frac {e_s^q} {\prod_{t\neq s}(e_s-e_t)}+
\sum^p_{s=1}\frac {e_s^q} {\prod_{t\neq s}(e_s-e_t)}=0
\end{align*}
The last equality is obtained from the induction hypotheses.
\end{proof}

\begin{lem}\label{lem:poly}
Let $f(x)$ be any polynomial of degree $\nu-2$, then $\sum^\nu_{s=1}\frac {f(e_s)} {\prod_{t\neq s}(e_s-e_t)}=0$. 
\end{lem}
This can be easily derived from Lemma \ref{lem:order}. 
Now we are ready to prove that $\hmk\vec{\beta^\ast}=0$. Note that the degree of the numerator of $\beta^\ast_r$ is $2k-3$ (see Equation~(\ref{equ:mainbeta})). Let $[\hmk]_i$ denote the $i$-th row of $\hmk$. We have the following calculations.
\begin{align*}
[\hmk]_1\vec{\beta^\ast}=&\sum_{r=1}^{2k}\frac {\prod^{2k-3}_{p=1} (p e_r+2k-2-p)} {\prod_{q\neq r} (e_r-e_q)}=0\\
[\hmk]_2\vec{\beta^\ast}=&\sum_{r=1}^{2k}\frac {\prod^{2k-3}_{p=1}e_r(p e_r+2k-2-p)} {\prod_{q\neq r} (e_r-e_q)}=0\end{align*}
For any $2< i\le 2k-1$, we have
\begin{align*}
&[\hmk]_i\vec{\beta^\ast}\\
=&\sum_{r=1}^{2k}\frac {e_r(1-e_r)^{i-2}} {\prod^{i-2}_{p=1}(p e_r+2k-2-p)}\frac {\prod^{2k-3}_{p=1} (p e_r+2k-2-p)} {\prod_{q\neq r} (e_r-e_q)}\\
=&\sum_{r=1}^{2k}\frac {e_r(1-e_r)^{i-2}\prod^{2k-3}_{p=i-1}(p e_r+2k-2-p)} {\prod_{q\neq r} (e_r-e_q)}=0
&\text{(Lemma~\ref{lem:poly})}
\end{align*}
The last equality is obtained by letting $v=2k-2$ in Lemma~\ref{lem:poly}. Therefore, $\hmk\vec{\beta^\ast}=0$. Note that $\vec{\beta^\ast}$ is also the solution for less than $2k-1$ alternatives. The theorem follows after applying Lemma~\ref{lem:exist}. 

\end{proof}

\begin{thm}\label{thm:id}
For $k=2$, and any $m \ge 4$, the $2$-PL is identifiable. %
\end{thm}
\begin{proof}
We will apply Lemma~\ref{lem:rel} to prove the theorem. That is, we will show that for all non-degenerate $\vec\theta^{(1)}, \vec\theta^{(2)},\vec\theta^{(3)},\vec\theta^{(4)}$ such that $\rank(\fmt)=4$. We recall that $\fmt$ is a $24\times 4$ matrix. Instead of proving $\rank(\fmt)=4$ directly, we will first obtain a $4\times 4$ matrix $\mathbf F^*=T\times \fmt$ by linearly combining some row vectors of $\fmt$ via a $4\times 24$ matrix $T$. Then, we show that $\rank(\mathbf{F^*})=4$, which implies that $\rank(\fmt)=4$.

For simplicity we use $[e_r,b_r,c_r,d_r]^\top$ to denote the parameter of $r$-th Plackett-Luce component for $a_1, a_2,a_3,a_4$ respectively. Namely,
$
\begin{bmatrix} \vec\theta^{(1)} & \vec\theta^{(2)} & \vec\theta^{(3)} & \vec\theta^{(4)}\end{bmatrix}=\begin{bmatrix}
e_1 & e_2 & e_3 & e_4\\
b_1 & b_2 & b_3 & b_4\\
c_1 & c_2 & c_3 & c_4\\
d_1 & d_2 & d_3 & d_4
\end{bmatrix}= \begin{bmatrix}
\vec\omega^{(1)}\\
\vec\omega^{(2)}\\
\vec\omega^{(3)}\\
\vec\omega^{(4)}
\end{bmatrix}
$, where for each $r\le 4$, $\vec\omega^{(r)}$ is a row vector. We further let $\vec 1= [1, 1, 1,1]$.

Clearly we have $\sum_{i=1}^4\vec\omega^{(i)}=\vec 1$. Therefore, if there exist three $\vec \omega$'s, for example $\{\vec\omega^{(1)},\vec\omega^{(2)},\vec\omega^{(3)}\}$, such that $\vec\omega^{(1)},\vec\omega^{(2)},\vec\omega^{(3)}$ and $\vec 1$ are linearly independent, then $\rank(\fmt)=4$ because each $\vec\omega^{(i)}$ corresponds to the probability of $a_i$ being ranked at the top, which means that $\vec\omega^{(i)}$ is a linear combination of rows in $\fmt$.
Because $\vec\theta^{(1)}, \vec\theta^{(2)},\vec\theta^{(3)},\vec\theta^{(4)}$  is non-degenerate, at least one of $\{\vec\omega^{(1)},\vec\omega^{(2)},\vec\omega^{(3)},\vec\omega^{(4)}\}$ is linearly independent of $\vec 1$. W.l.o.g.~suppose $\vec\omega^{(1)}$ is linearly independent of $\vec{1}$. This means that not all of $e_1, e_2, e_3, e_4$ are equal. The theorem will be proved in the following two cases.

\noindent\textbf{Case 1.} $\vec\omega^{(2)}$, $\vec\omega^{(3)}$, and $\vec\omega^{(4)}$ are all linear combinations of $\vec{1}$ and $\vec\omega^{(1)}$. 

\noindent{\bf Case 2.} There exists a $\vec\omega^{(i)}$ (where $i\in \{2,3,4\}$) that is linearly independent of $\vec{1}$ and $\vec\omega^{(1)}$.

\noindent\textbf{Case 1.} For all $i=2,3,4$ we can rewrite $\vec\omega^{(i)}=p_i\vec\omega^{(1)}+q_i$ for some constants $p_i, q_i$. 
More precisely, for all $r=1,2,3,4$ we have:
\begin{align}
b_r &= p_2 e_r+q_2\label{eq:cc11}\\
c_r &= p_3 e_r+q_3\label{eq:cc12}\\
d_r &= p_4 e_r+q_4\label{eq:cc13}
\end{align}
Because $\vec\omega^{(1)}+\vec\omega^{(2)}+\vec\omega^{(3)}+\vec\omega^{(4)}=\vec 1$, we have 
\begin{align}
p_2+p_3+p_4&=-1\label{eq:cond1}\\
q_2+q_3+q_4&=1\label{eq:cond2}
\end{align}
In this case for each $r\leq 4$, the $r$-th column of $\fmt$, which is $\vec f_4(\vec \theta^{(r)})$, is a function of $e_r$. Because the $\vec\theta$'s are non-degenerate, $e_1,e_2,e_3,e_4$ must be pairwise different.

We assume $p_2\neq 0$ and $q_2\neq 1$ for all subcases of {\bf Case 1} (This will be denoted as {\bf Case 1 Assumption}). The following claim shows that there exists $p_i, q_i$ where $i\in\{2,3,4\}$ satisfying this condition. If $i\neq 2$ we can switch the row of alternatives $a_2$ and $a_i$. Then the assumption holds.

\begin{claim}\label{topclaim}
There exists $i\in{2,3,4}$ which satisfy the following conditions:
\begin{itemize}
\item $q_i\neq 1$
\item $p_i\neq 0$
\end{itemize}
\end{claim}
\begin{proof}
Suppose for all $i=2,3,4$, $q_i=1$ or $p_i=0$. 

If $p_i=0$, $q_i$ must be positive because $b_r,c_r,d_r$ are all positive. If $p_i\neq 0$, Then $q_i=1$ due to the assumption above. So $q_i>0$ for all $i=2,3,4$. If there exists $i$ s.t. $q_i=1$, then \eqref{eq:cond2} does not hold. So for all $i$, $q_i\neq 1$. Then $p_i=0$ holds for all $i\in\{2,3,4\}$, which violates \eqref{eq:cond1}. 
\end{proof}

{\bf Case 1.1}: $p_2+q_2\neq 0$ and $p_2+q_2\neq 1$.

For this case we first define a $4\times 4$ matrix $\mathbf{\hat F}$ as in Table~\ref{tabfa}. 
\begin{table}[htp]
\begin{center}
\begin{tabular}{|c|c|}
\hline
$\mathbf{\hat F}$ & Moments\\
\hline
\multirow{4}{*}{
$\begin{bmatrix}
1 & 1 & 1 & 1\\
e_1 & e_2 & e_3 & e_4\\
\frac {e_1b_1} {1-b_1} & \frac {e_2b_2} {1-b_2} & \frac {e_3b_3} {1-b_3} & \frac {e_4b_4} {1-b_4}\\
\frac {e_1b_1} {1-e_1} & \frac {e_2b_2} {1-e_2} & \frac {e_3b_3} {1-e_3} & \frac {e_4b_4} {1-e_4} \end{bmatrix}$}
& $\vec 1$\\
& $a_1\succ\text{others}$\\
& $a_2\succ a_1\succ\text{others}$\\
& $a_1\succ a_2\succ\text{others}$\\
& \\
\hline
\end{tabular}
\end{center}
\caption{$\mathbf{\hat F}$.\label{tabfa}}
\end{table}%
We use $\vec{1}$ and $\vec\omega^{(1)}$ to denote the first two rows of $\hat F$.  $\vec\omega^{(1)}$ corresponds to the probability that $a_1$ is ranked at the top. We call such a probability a {\em moment}. Each moment is the sum of probabilities of some rankings. For example, the ``$a_1\succ\text{others}$'' moment is the total probability for $\{V\in{\mathcal L}(\mathcal A): a_1\text{ is ranked at the top of }V\}$. It follows that there exists a $4\times 24$ matrix $\hat T$ such that $\mathbf{\hat F}=\hat T\times  \fmt$.

Define 
\begin{align*}
\vec\theta^{(b)}&=[\frac 1 {1-b_1}, \frac 1 {1-b_2}, \frac 1 {1-b_3}, \frac 1 {1-b_4}]\\
&=[\frac 1 {1-p_2 e_1-q_2}, \frac 1 {1-p_2 e_2-q_2}, \frac 1 {1-p_2 e_3-q_2}, \frac 1 {1-p_2 e_4-q_2}]\\
\end{align*}
and
\begin{align}
\vec\theta^{(e)}&=[\frac 1 {1-e_1}, \frac 1 {1-e_2}, \frac 1 {1-e_3}, \frac 1 {1-e_4}]\label{eq:thetaa}
\end{align}
and let
$
\mathbf{F}^\ast=
\begin{bmatrix}
\vec 1\\
\vec\omega^{(1)}\\
\vec\theta^{(b)}\\
\vec\theta^{(e)}
\end{bmatrix}
$. It can be verified that $\mathbf{\hat F}=T^*\times\mathbf{F}^\ast$, where
\begin{equation*}
T^*=\begin{bmatrix}
1 & 0 & 0 & 0 \\
0 & 1 & 0 & 0 \\
-\frac 1 {p_2} & -1 & \frac {1-q_2} {p_2} & 0 \\
-(p_2+q_2) & -p_2 & 0 & p_2+q_2 \\
\end{bmatrix}
\end{equation*}
Because Case 1.1 assumes that $p_2+q_2\ne 0$ and we can select $a_2$ such that $p_2\neq 0$, $q_2\neq 1$, we have that $T^*$ is invertible. Therefore,  $\mathbf{F^*}=(T^*)^{-1}\times\mathbf{\hat F}$, which means that $\mathbf{F^*}=T\times \fmt$ for some $4\times 24$ matrix $T$.

We now prove that $\rank(\mathbf{F^*})=4$. For the sake of contradiction, suppose that $\rank(\mathbf{F^*})<4$. It follows that there exist a nonzero row vector $\vec{t}=[t_1, t_2, t_3, t_4]$, such that 
$
\vec{t}\cdot\mathbf{F}^\ast=0
$. 
This means that for all $r\leq 4$,
\begin{equation*}
t_1+t_2 e_r+\frac {t_3} {1-p_2 e_r-q_2}+\frac {t_4} {1-e_r}=0
\end{equation*} 
Let $
f(x)=t_1+t_2x+\frac {t_3} {1-p_2x-q_2}+\frac {t_4} {1-x}
$. Let $g(x)=(1-p_2x-q_2)(1-x)f(x)$. We recall that $e_1,e_2,e_3,e_4$ are four roots of $f(x)$, which means that they are also the four roots of $g(x)$. Because in Case 1.1 we assume that $p_2+q_2\neq 1$, it can be verified that not all coefficients of $g(x)$ are zero. We note that the degree of $g(x)$ is $3$. Therefore, due to the Fundamental Theorem of Algebra, $g(x)$ has at most three different roots. This means that $e_1,e_2,e_3,e_4$ are not pairwise different, which is a contradiction.

Therefore, $\rank(\mathbf{F^*})=4$, which means that $\rank(\fmt)=4$. This finishes the proof for Case 1.1. 

\noindent{\bf Case 1.2.} $p_2+q_2=1$. 

If we can find an alternative $a_i$, such that $p_i$ and $q_i$ satisfy the following conditions:
\begin{itemize}
\item $p_i\neq 0$
\item $q_i\neq 1$
\item $p_i+q_i\neq 0$
\item $p_i+q_i\neq 1$
\end{itemize}
Then we can use $a_i$ as $a_2$, which belongs to Case 1.1. Otherwise we have the following claim.
\begin{claim}\label{c12cond}
If for $i\in\{3,4\}$, $p_i$ and $q_i$ satisfy one of the following conditions
\begin{enumerate}
\item $p_i=0$
\item $p_i\neq 0$, $q_i=1$
\item $p_i+q_i=0$
\item $p_i+q_i=1$
\end{enumerate}
We claim that there exists $i\in\{3,4\}$ s.t. $p_i$, $q_i$ satisfy condition 2, namely $p_i\neq 0$, $q_i=1$. 
\end{claim}
\begin{proof}
Suppose $p_i=0$, then $q_i>0$ because $p_i e_1+q_i$ is a parameter in a Plackett-Luce component. If for $i=3,4$, $p_i$ and $q_i$ satisfy any of conditions 1, 3 or 4, then $q_i\geq -p_i$ ($q_i>0$ for condition 1, $q_i=-p_i$ for condition 3, $q_i=1-p_i>-p_1$ for condition 4). As $\sum^4_{i=2}p_i=-1$, $\sum^4_{i=2}q_i\geq1-\sum^4_{i=2}p_i=2$, which contradicts that $\sum^4_{i=2}q_i=1$. 
\end{proof}

Without loss of generality we let $p_3\neq 0$ and $q_3=1$. We now construct $\hat{\mathbf{F}}$ as is shown in the following table. 
\begin{table}[htp]
\begin{center}
\begin{tabular}{|c|c|}
\hline
$\mathbf{\hat F}$ & Moments\\
\hline
\multirow{4}{*}{
$
\begin{bmatrix}
1 & 1 & 1 & 1\\
e_1 & e_2 & e_3 & e_4\\
\frac {e_1b_1} {1-e_1} & \frac {e_2b_2} {1-e_2} & \frac {e_3b_3} {1-e_3} & \frac {e_4b_4} {1-e_4}\\
\frac {c_1b_1} {1-c_1} & \frac {c_2b_2} {1-c_2} & \frac {c_3b_3} {1-c_3} & \frac {c_4b_4} {1-c_4}\\\end{bmatrix}$}
& $\vec 1$\\
& $a_1\succ\text{others}$\\
& $a_1\succ a_2\succ\text{others}$\\
& $a_3\succ a_2\succ\text{others}$\\
& \\
\hline
\end{tabular}
\end{center}
\label{tabfb}
\end{table}%

We define $\vec\theta^{(b)}$ the same way as in {\bf Case 1.1}, and define
\begin{equation*}
\vec\theta^{(c)}=[\frac 1 {e_1}, \frac 1 {e_2}, \frac 1 {e_3}, \frac 1 {e_4}]
\end{equation*}
Define
\begin{equation*}%
\mathbf{F}^\ast=
\begin{bmatrix}
\vec 1\\
\vec\omega^{(1)}\\
\vec\theta^{(e)}\\
\vec\theta^{(c)}
\end{bmatrix}
\end{equation*}
We will show that $\mathbf{\hat F}=T^*\times\mathbf{F}^\ast$ where $T^*$ has full rank. 

For all $r=1,2,3,4$
\begin{align*}
\frac {c_rb_r} {1-c_r}&=\frac {(p_3 e_r+q_3)(p_2 e_r+q_2)} {1-p_3 e_r-q_3}=\frac {(p_3 e_r+1)(p_2 e_r+1-p_2)} {-p_3 e_r}=-p_2 e_r+(p_2-1-\frac {p_2} {p_3})-\frac {1-p_2} {p_3 e_r}
\end{align*}
So
\begin{equation*}%
\mathbf{\hat F}=
\begin{bmatrix}
\vec 1\\
\vec\omega^{(1)}\\
-\vec 1-p_2\vec\omega^{(1)}+\vec\theta^{(e)}\\
(p_2-1-\frac {p_2} {p_3})\vec 1-p_2\vec\omega^{(1)}-\frac {1-p_2} {p_3}\vec\theta^{(c)}
\end{bmatrix}
\end{equation*}

Suppose $p_2\neq 1$, we have $\mathbf{\hat F}=T^*\times\mathbf{F}^\ast$ where
\begin{equation*}
T^*=\begin{bmatrix}
1 & 0 & 0 & 0 \\
0 & 1 & 0 & 0 \\
-1 & -p_2 & 1 & 0\\
p_2-1-\frac {p_2} {p_3} & -p_2 & 0 & -\frac {1-p_2} {p_3}
\end{bmatrix}
\end{equation*}
which is full rank. So $\operatorname{rank}(\mathbf{F}^\ast)=\operatorname{rank}(\mathbf{\hat F})$. 

If rank$(\mathbf{F}^{2}_4)\leq 3$, then there is at least one column in $\mathbf{F}^2_4$ dependent of the other columns. As all rows in $\hat{\mathbf F}$ are linear combinations of rows in $\mathbf{F}^{2}_4$, there is also at least one column in $\hat{\mathbf{F}}$ dependent of the other columns. Therefore we have rank$(\mathbf{\hat F})\leq 3$. Further we have rank$(\mathbf{F}^\ast)\leq 3$.
Therefore, there exists a nonzero row vector $\vec{t}=[t_1, t_2, t_3, t_4]$, s.t. 
\begin{equation*}
\vec{t}\mathbf{F}^\ast=0
\end{equation*}
Namely, for all $r\leq 4$,
\begin{equation*}%
t_1+t_2 e_r+\frac {t_3} {1-e_r}+\frac {t_4} {e_r}=0
\end{equation*}
Let
\begin{align*}
f(x)&=t_1+t_2x+\frac {t_3} {1-x}+\frac {t_4} {x}=0\\
g(x)&=x(1-x)f(x)=x(1-x)(t_1+t_2)+t_3x+t_4(1-x)
\end{align*}
If any of the coefficients in $f(x)$ is nonzero, then $g(x)$ is a polynomial of degree at most 3. There will be a maximum of $3$ different roots. Since this equation holds for $e_r$ where $r=1,2,3,4$, there exists $s\neq t$ s.t. $e_s=e_t$. Otherwise $g(x)=f(x)=0$ for all $x$. We have
\begin{align*}
g(0)&=t_4=0\\
g(1)&=t_3=0
\end{align*}
Substitute $t_3=t_4=0$ into $f(x)$, we have $f(x)=t_1+t_2x=0$ for all $x$. So $t_1=t_2=0$. This contradicts the nonzero requirement of $\vec t$. Therefore there exists $s\neq t$ s.t. $e_s=e_t$. From \eqref{eq:cc11}\eqref{eq:cc12}\eqref{eq:cc13} we have $\vec\theta^{(s)}=\vec\theta^{(t)}$, which is a contradition.

If $p_2=1$, from the assumption of {\bf Case 1.2} $q_2=0$. So $b_r=e_r$ for $r=1,2,3,4$. Then from \eqref{eq:cond1} we have $p_4=-p_3-2$ and from \eqref{eq:cond2} we have $q_4=0$. Since $p_4$ and $q_4$ satisfy one of the four conditions in Claim \ref{c12cond}, we can show it must satisfy Condition 4. ($q_4=0$ violates Condition 2. If it satisfies Condition 1 or 3, then $p_4=0$. Then $d_r=p_4a_r+q_4=0$, which is impossible.) So $p_4=1$, and $p_3=-3$. This is the case where $\vec\omega^{(1)}=\vec\omega^{(2)}=\vec\omega^{(4)}$ and $\vec\omega^{(3)}=1-3\vec\omega^{(1)}$. For this case, we use $a_3$ as $a_1$. After the transformation, we have $\vec\omega^{(2)}=\vec\omega^{(3)}=\vec\omega^{(4)}=\frac {1-\vec\omega^{(1)}} 3$. We claim that this lemma holds for a more general case where $p_i+q_i=0$ for $i=2,3,4$. It is easy to check that $p_i=-\frac 1 3$ and $q_i=\frac 1 3$ belongs to this case.
\begin{claim}\label{case13c}
For all $r=1,2,3,4$, if
\begin{equation}\label{claim13set}
\vec\theta^{(r)}=\begin{bmatrix}e_r\\
b_r\\
c_r\\
d_r
\end{bmatrix}=\begin{bmatrix} e_r\\
p_2 e_r-p_2\\
p_3 e_r-p_3\\
-(1+p_2+p_3)e_r+(1+p_2+p_3)
\end{bmatrix}
\end{equation}
The model is identifiable.
\end{claim}
\begin{proof}
We first show a claim, which is useful to the proof.
\begin{claim}\label{condp}
Under the settings of \eqref{claim13set}, $-1<p_2,p_3<0$, $-1<p_2+p_3<0$.
\end{claim}
\begin{proof}
From the definition of Plackett-Luce model, $0<e_r,b_r,c_r,d_r<1$. From \eqref{claim13set}, we have $p_2=\frac {b_r} {e_r-1}$. Since $b_r>0$ and $e_r<1$, $p_2<0$. Similarly we have $p_3<0$ and $-(1+p_2+p_3)<0$. So $-1<p_2+p_3<0$. Then we have $p_2>-1-p_3$. So $-1-p_3<p_2<0$, $p_3>-1$. Similarly we have $p_2>-1$. 

\end{proof}

In this case, we construct $\hat{\mathbf F}$ in the following way.
\begin{table}[htp]
\begin{center}
\begin{tabular}{|c|c|}
\hline
$\mathbf{\hat F}$ & Moments\\
\hline
\multirow{4}{*}{
$\begin{bmatrix}
1 & 1 & 1 & 1\\
e_1 & e_2 & e_3 & e_4\\
\frac {e_1b_1} {1-b_1} & \frac {e_2b_2} {1-b_2} & \frac {e_3b_3} {1-b_3} & \frac {e_4b_4} {1-b_4}\\
\frac {e_1b_1c_1} {(1-b_1)(1-b_1-c_1)} & \frac {e_2b_2c_2} {(1-b_2)(1-b_2-c_2)} & \frac {e_3b_3c_3} {(1-b_3)(1-b_3-c_3)} & \frac {e_4b_4c_4} {(1-b_4)(1-b_4-c_4)}
\end{bmatrix}$}
& $\vec 1$\\
& $a_1\succ\text{others}$\\
& $a_2\succ a_1\succ\text{others}$\\
& $a_2\succ a_3\succ a_1\succ a_4$\\
& \\
\hline
\end{tabular}
\end{center}
\label{tabfc}
\end{table}%

Define $\vec\theta^{(b)}$ the same way as in {\bf Case 1.1}
\begin{align*}
\vec\theta^{(b)}&=[\frac 1 {1-b_1}, \frac 1 {1-b_2}, \frac 1 {1-b_3}, \frac 1 {1-b_4}]\\
&=[\frac 1 {1-p_2 e_1+p_2}, \frac 1 {1-p_2 e_2+p_2}, \frac 1 {1-p_2 e_3+p_2}, \frac 1 {1-p_2 e_4+p_2}]
\end{align*}
And define
\begin{align*}
\vec\theta^{(bc)}=&[\frac 1 {1-(p_2+p_3)e_1+p_2+p_3}, \frac 1 {1-(p_2+p_3)e_2+p_2+p_3}, \\
&\frac 1 {1-(p_2+p_3)e_3+p_2+p_3}, \frac 1 {1-(p_2+p_3)e_4+p_2+p_3}]
\end{align*}
Further define
\begin{equation*}%
\mathbf{F}^\ast=
\begin{bmatrix}
\vec 1\\
\vec\omega^{(1)}\\
\vec\theta^{(b)}\\
\vec\theta^{(bc)}
\end{bmatrix}
\end{equation*}
We will show $\hat{\mathbf F}=T^*\times\mathbf{F}^\ast$ where $T^*$ has full rank.

The last two rows of $\hat{\mathbf F}$
\begin{align*}
\frac {e_rb_r} {1-b_r}&=-e_r-\frac 1 {p_2}+\frac {1+p_2} {p_2(1-p_2 e_r+p_2)}\\
\frac {e_rb_rc_r} {(1-b_r)(1-b_r-c_r)}&=\frac {e_r(p_2 e_r-p_2)(p_3 e_r-p_3)} {(1-p_2 e_r+p_2)(1-(p_2+p_3)e_r+p_2+p_3)}\\
&=\frac {p_2p_3e_r(e_r-1)^2} {(1-p_2 e_r+p_2)(1-(p_2+p_3)e_r+p_2+p_3)}\\
&=\frac {p_3(2p_2+p_3)} {p_2(p_2+p_3)^2}+\frac {p_3} {p_2+p_3}e_r-\frac {(1+p_2)} {p_2(1-p_2 e_r+p_2)}\\
&+\frac {p_2(1+p_2+p_3)} {(1-(p_2+p_3)e_r+p_2+p_3)(p_2+p_3)^2}
\end{align*}
So
\begin{equation*}%
\mathbf{\hat F}=
\begin{bmatrix}
\vec 1\\
\vec\omega^{(1)}\\
-\frac 1 {p_2}\vec 1-\vec\omega^{(1)}+\frac {1+p_2} {p_2}\vec\theta^{(b)}\\
\frac {p_3(2p_2+p_3)} {p_2(p_2+p_3)^2}\vec 1+\frac {p_3} {p_2+p_3}\vec\omega^{(1)}-\frac {1+p_2} {p_2}\vec\theta^{(b)}+\frac {p_2(1+p_2+p_3)} {(p_2+p_3)^2}\vec\theta^{(bc)}
\end{bmatrix}
\end{equation*}

Then we have $\mathbf{\hat F}=T^*\times\mathbf{F}^\ast$ where
\begin{equation*}
T^*=\begin{bmatrix}
1 & 0 & 0 & 0 \\
0 & 1 & 0 & 0 \\
-\frac 1 {p_2} & -1 & \frac {1+p_2} {p_2} & 0 \\
\frac {p_3(2p_2+p_3)} {p_2(p_2+p_3)^2} & \frac {p_3} {p_2+p_3} & -\frac {1+p_2} {p_2} & \frac {p_2(1+p_2+p_3)} {(p_2+p_3)^2}\\
\end{bmatrix}
\end{equation*}
From Claim \ref{condp}, we have $-1<p_2<0$ and $-1<p_2+p_3<0$, so $\frac {1+p_2} {p_2}\neq 0$ and $\frac {p_2(1+p_2+p_3)} {(p_2+p_3)^2}\neq 0$. So $T$ has full rank. Then $\operatorname{rank}(\mathbf{F}^\ast)=\operatorname{rank}(\mathbf{\hat F})$. 

If rank$(\mathbf{F}^{2}_4)\leq 3$, then there is at least one column in $\mathbf{F}^2_4$ dependent of other columns. As all rows in $\hat{\mathbf F}$ are linear combinations of rows in $\mathbf{F}^2_4$, rank$(\mathbf{\hat F})\leq 3$. Since $\operatorname{rank}(\mathbf{F}^\ast)=\operatorname{rank}(\mathbf{\hat F})$, we have rank$(\mathbf{F}^\ast)\leq 3$. Therefore, there exists a nonzero row vector $\vec{t}=[t_1, t_2, t_3, t_4]$, s.t. 
\begin{equation*}
\vec{t}\mathbf{F}^\ast=0
\end{equation*}
Namely, for all $r\leq 4$,
\begin{equation*}%
t_1+t_2 e_r+\frac {t_3} {1-p_2a_r+p_2}+\frac {t_4} {1-(p_2+p_3)e_r+p_2+p_3}=0
\end{equation*}
Let
\begin{align*}
f(x)&=t_1+t_2x+\frac {t_3} {1-p_2x+p_2}+\frac {t_4} {1-(p_2+p_3)x+p_2+p_3}\\
g(x)&=(1-p_2x+p_2)(1-(p_2+p_3)x+p_2+p_3)(t_1+t_2x)\\
&+t_3(1-(p_2+p_3)x+p_2+p_3)+t_4(1-p_2x+p_2)
\end{align*}
If any of the coefficients of $g(x)$ is nonzero, then $g(x)$ is a polynomial of degree at most 3. There will be a maximum of $3$ different roots. As the equation holds for all $e_r$ where $r=1,2,3,4$. There exists $s\neq t$ s.t. $e_s=e_t$. Otherwise $g(x)=f(x)=0$ for all $x$. We have
\begin{align*}
g(\frac {1+p_2} {p_2})&=\frac {-t_3p_3} {p_2}=0\\
g(\frac {1+p_2+p_3} {p_2+p_3})&=\frac {t_4p_3} {p_2+p_3}=0
\end{align*}
From Claim \ref{condp} we know $p_2, p_3<0$ and $p_2+p_3<0$. So $t_3=t_4=0$. Substitute it into $f(x)$ we have $f(x)=t_1+t_2x=0$ for all $x$. So $t_1=t_2=0$. This contradicts the nonzero requirement of $\vec t$. Therefore there exists $s\neq t$ s.t. $e_s=e_t$. According to \eqref{eq:cc11}\eqref{eq:cc12}\eqref{eq:cc13} we have $\vec\theta^{(s)}=\vec\theta^{(t)}$, which is a contradition.

\end{proof}

\noindent{\bf Case 1.3.} $p_2+q_2=0$. 

If there exists $i$ such that $p_i+q_i=1$, then we can use $a_i$ as $a_2$ and the proof is done in Case 1.2. It may still be possible to find another $i$ such that $p_i,q_i$ satisfy the following two conditions:
\begin{enumerate}
\item $p_i\neq 0$ and $q_i\neq 1$;
\item $p_i+q_i\neq 0$.
\end{enumerate}
If we can find another $i$ to satisfy the two conditions, then the proof is done in Case 1.1. Then we can proceed by assuming that the two conditions are not satisfied by any $i$. We will prove that the only possibility is $p_i+q_i=0$ for $i=2,3,4$.

Suppose for $i=3,4$, $p_i$ and $q_i$ violate Condition 1. If $p_i=0$, then $q_i>0$. If at least one of them has $q_i=1$, then $e_r+b_r+c_r+d_r>1$, which is impossible. If both alternatives violates Condition 1 and $p_3=p_4=0$, then $0<q_3, q_4<1$. According to \eqref{eq:cond1} $p_2=-1$. As $p_2+q_2=0$, we have $q_2=1$. From \eqref{eq:cond2}, $q_3+q_4=2$, which is impossible. So there exists $i\in\{3,4\}$ such that $p_i+q_i=0$. Then from $\sum_i\theta^r_i=1$ we obtain the only case we left out, which is
\begin{align*}
e_r&\\
b_r&=p_2 e_r-p_2\\
c_r&=p_3 e_r-p_3\\
d_r&=-(1+p_2+p_3) e_r+(1+p_2+p_3)
\end{align*}
This case has been proved in Claim \ref{case13c}.

\noindent\textbf{Case 2}: There exists $\vec\omega^{(i)}$ that is linearly independent of $\vec{1}$ and $\vec\omega^{(1)}$. W.l.o.g. let it be
$\vec\omega^{(2)}$. Define matrix
 
\begin{equation*}
\mathbf{G}=\begin{bmatrix}
\vec 1\\
\vec\omega^{(1)}\\
\vec\omega^{(2)}
\end{bmatrix}=
\begin{bmatrix}
1 & 1 & 1 & 1\\
e_1 & e_2 & e_3 & e_4\\
b_1 & b_2 & b_3 & b_4
\end{bmatrix}
\end{equation*}
The rank of $\mathbf{G}$ is 3. Since $\mathbf G$ is constructed using linear combinations of rows in $\mathbf{F}^2_4$, the rank of $\mathbf{F}^2_4$ is at least 3. 

If $\vec\omega^{(3)}$ or $\vec\omega^{(4)}$ is independent of rows in $\mathbf{G}$, then we can append it to $\mathbf{G}$ as the fourth row so that the rank of the new matrix is $4$. Then $\mathbf F^2_4$ is full rank. So we only need to consider the case where $\vec\omega^{(3)}$ and $\vec\omega^{(4)}$ are linearly dependent of $\vec 1$, $\vec\omega^{(1)}$, and $\vec\omega^{(2)}$. Let 
\begin{align}
\vec\omega^{(3)}&=x_3\vec\omega^{(1)}+y_3\vec\omega^{(2)}+z_3\vec 1\label{cc21}\\
\vec\omega^{(4)}&=x_4\vec\omega^{(1)}+y_4\vec\omega^{(2)}+z_4\vec 1\label{cc22}
\end{align}
where $x_3+x_4=-1$, $y_3+y_4=-1$, $z_3+z_4=1$.
\begin{claim}
There exists $i\in\{3,4\}$ such that $x_i+z_i\neq 0$. 
\end{claim}
\begin{proof}
If in the current setting $\exists i\in\{3,4\}$ s.t. $x_i+z_i\neq 0$, then the proof is done. If in the current setting $x_3+z_3=x_4+z_4=0$, but $\exists i\in\{3,4\}$ s.t. $y_i+z_i=0$, then we can switch the role of $e_r$ and $b_r$, namely
\begin{align*}
\vec\omega^{(3)}&=y_3\vec\omega^{(1)}+x_3\vec\omega^{(2)}+z_3\vec 1\\
\vec\omega^{(4)}&=y_4\vec\omega^{(1)}+x_4\vec\omega^{(2)}+z_4\vec 1
\end{align*}
Then the proof is done. If for all $i\in\{3,4\}$ we have $x_i+z_i=0$ and $y_i+z_i=0$, then we switch the role of $e_r$ and $c_r$ and get
\begin{align*}
\vec\omega^{(3)}&=\frac 1 {x_3}(\vec\omega^{(1)}-y_3\vec\omega^{(2)}-z_3\vec 1)\\
\vec\omega^{(4)}&=\frac 1 {x_4}(\vec\omega^{(1)}-y_4\vec\omega^{(2)}-z_4\vec 1)
\end{align*}
If $\frac {1-z_3} {x_3}\neq 0$, namely $z_3\neq 1$, the proof is done. Suppose $z_3=1$, then $x_3=y_3=-1$. We have $\vec\omega^{(3)}=1-\vec\omega^{(1)}-\vec\omega^{(2)}$. Then $\vec\omega^{(4)}=\vec 0$, which is impossible.
\end{proof}

Without loss of generality let $x_3+z_3\neq 0$. Similar to the previous proofs, we want to construct a matrix $\mathbf{G}'$ using linear combinations of rows from $\mathbf{F}^2_4$. Let the first 3 rows for $\mathbf{G}'$ to be $\mathbf{G}$. Then rank$(\mathbf G')\geq 3$. Since rank$(\mathbf{F}^2_4)\leq 3$ and all rows in $\mathbf{G}'$ are linear combinations of rows in $\mathbf{F}^2_4$, we have rank$(\mathbf{G}')\leq 3$. So rank$(\mathbf{G}')=3$. This means that any linear combinations of rows in $\mathbf{F}^2_4$ is linearly dependent of rows in $\mathbf{G}$. 

Consider the moment where $a_1$ is ranked at the top and $a_2$ is ranked at the second position. Then $[\frac {e_1b_1} {1-e_1}, \frac {e_2b_2} {1-e_2}, \frac {e_3b_3} {1-e_3}, \frac {e_4b_4} {1-e_4}]$ is linearly dependent of $\mathbf{G}$. Adding $\vec\omega^{(2)}$ to it, we have
\begin{equation*}
\vec\theta^{(eb)}=[\frac {b_1} {1-e_1}, \frac {b_2} {1-e_2}, \frac {b_3} {1-e_3}, \frac {b_4} {1-e_4}]
\end{equation*}
which is linearly dependent of $\mathbf{G}$. 

Similarly consider the moment that $a_1$ is ranked at the top and $a_3$ is ranked at the second position. We obtain $[\frac {e_1c_1} {1-e_1}, \frac {e_2c_2} {1-e_2}, \frac {e_3c_3} {1-e_3}, \frac {e_4c_4} {1-e_4}]$. Add $\vec\omega^{(3)}$ to it, we get
\begin{equation*}
\vec\theta^{(ec)}=[\frac {c_1} {1-e_1}, \frac {c_2} {1-e_2}, \frac {c_3} {1-e_3}, \frac {c_4} {1-e_4}]
\end{equation*}
which is linearly dependent of $\mathbf{G}$. 

Recall from \eqref{eq:thetaa}
\begin{equation*}
\vec\theta^{(e)}=[\frac 1 {1-e_1}, \frac 1 {1-e_2}, \frac 1 {1-e_3}, \frac 1 {1-e_4}]
\end{equation*}

Then
\begin{align*}
\vec\theta^{(ec)}&=[\frac {x_3 e_1+y_3b_1+z_3} {1-e_1}, \frac {x_3 e_2+y_3b_2+z_3} {1-e_2}, \frac {x_3 e_3+y_3b_3+z_3} {1-e_3}, \frac {x_3 e_4+y_3b_4+z_3} {1-e_4}]\\
&=(x_3+z_3)\vec\theta^{(e)}+y_3\vec\theta^{(eb)}-x_3\vec 1
\end{align*}
Because both $\vec\theta^{(eb)}$ and $\vec\theta^{(ec)}$ are linearly dependent of $\mathbf{G}$, $\vec\theta^{(e)}$ is also linearly dependent of $\mathbf{G}$. Make it the 4th row of $\mathbf{G}'$. Suppose the rank of $\mathbf{G}'$ is still $3$. We will first prove this lemma under the assumption below, and then discuss the case where the assumption does not hold.

Assumption 1: Suppose $\vec{1}, \vec\omega^{(1)}, \vec\theta^{(e)}$ are linearly independent. 

Then $\vec\omega^{(2)}$ is a linear combination of $\vec{1}, \vec\omega^{(1)}$ and $\vec\theta^{(e)}$. We write $\vec\omega^{(2)}=s_1+s_2\vec\omega^{(1)}+s_3\vec\theta^{(e)}$ for some constants $s_1,s_2,s_3$. We have $s_3\neq 0$ because $\vec\omega^{(2)}$ is linearly independent of $\vec 1$ and $\vec\omega^{(1)}$. Elementwise, for $r=1,2,3,4$ we have

\begin{equation}\label{bex}
b_r=s_1+s_2 e_r+\frac {s_3} {1-e_r}
\end{equation} 

Let
\begin{equation*}
\mathbf{G}''=\begin{bmatrix}
\mathbf{G}\\
\vec\theta^{(eb)}
\end{bmatrix}
\end{equation*}
$\vec\theta^{(eb)}$ is linearly dependent of $\mathbf G$. There exists a non-zero vector $\vec{h}=[h_1, h_2, h_3, h_4]$ such that $\vec{h}\cdot\mathbf{G}''=0$. Namely $h_1\vec 1+h_2\vec\omega^{(1)}+h_3\vec\omega^{(2)}+h_4\vec\theta^{(eb)}=0$. Elementwise, for all $r=1,2,3,4$

\begin{equation}\label{thr}
h_1+h_2 e_r+h_3b_r+h_4\frac {b_r} {1-e_r}=0
\end{equation}
where $h_4\neq 0$ because otherwise rank$(\mathbf G)=2$. Substitute \eqref{bex} into \eqref{thr}, and multiply both sides of it by $(1-e_r)^2$, we get
\begin{align*}
(h_1+h_2 e_r+h_3b_r)(1-e_r)^2+h_4(s_1+s_2 e_r)(1-e_r)+h_4s_3=0
\end{align*}
Let
\begin{equation*}
f(x)=(h_1+h_2 e_r+h_3b_r)(1-e_r)^2+h_4(s_1+s_2 e_r)(1-e_r)+h_4s_3
\end{equation*}
We claim that not all coefficients of $x$ are zero, because $f(1)=h_4s_3\neq 0$ ($s_3\neq 0$ and $h_4\neq 0$ by assumption). Then there are a maximum of 3 different roots, each of which uniquely determines $b_r$ by (\ref{bex}). This means that there are at least two identical components. Namely $\exists s\neq t$ s.t. $\vec\theta^{(s)}=\vec\theta^{(t)}$.

If Assumption 1 does not hold, namely $\vec\theta^{(e)}$ is a linear combination of $\vec 1$ and $\vec\omega^{(1)}$, let
\begin{equation}\label{asp2}
\frac 1 {1-e_r}=p_5 e_r+q_5
\end{equation}
Define
\begin{equation*}
f(x)=\frac 1 {1-x}-p_5x-q_5
\end{equation*}
If $f(x)$ has only $1$ root or two identical roots between $0$ and $1$, then all columns of $\mathbf{G}$ have identical $e_r$-s. This means $\vec\omega^{(1)}$ is dependent of $\vec 1$, which is a contradiction. So we only consider the situation where $f(x)$ has two different roots between $0$ and $1$, denoted by $u_1$ and $u_2$ ($u_1\neq u_2$). Because $e_1, e_2, e_3, e_4$ are roots of $f(x)$, there must be at least two identical $e_r$'s, with the value $u_1$ or $u_2$. 

Substitute \eqref{asp2} into $\vec\theta^{(eb)}$, we have $\vec\theta^{(eb)}=[b_1(p_5 e_1+q_5), b_2(p_5 e_2+q_5), b_3(p_5 e_3+q_5), b_4(p_5 e_4+q_5)]$, which is linearly dependent of $\mathbf{G}$. So there exists nonzero vector $\vec{\gamma_1}=[\gamma_{11},\gamma_{12},\gamma_{13}, \gamma_{14}]$ such that
\begin{equation*}
\gamma_{11}+\gamma_{12}e_r+\gamma_{13}b_r+\gamma_{14}b_r(p_5 e_r+q_5)=0\label{sst1}
\end{equation*}
From which we get
\begin{equation}\label{exb}
(\gamma_{13}+\gamma_{14}p_5 e_r+\gamma_{14}q_5)b_r=-(\gamma_{11}+\gamma_{12}e_r)
\end{equation}
We recall that $e_r=u_1$ or $e_r=u_2$ for $r=1,2,3,4$. Since $u_1\neq u_2$, there exists $i\in\{1,2\}$ s.t. $\gamma_{13}+\gamma_{14}p_5u_i+\gamma_{14}q_5\neq 0$. W.l.o.g. let it be $u_1$. If at least two of the $e_r$'s are $u_1$, without loss of generality let $e_1=e_2=u_1$. Then using \eqref{exb} we know $b_1=b_2=\frac {-(\gamma_{11}+\gamma_{12}u_1)} {(\gamma_{13}+\gamma_{14}p_5u_1+\gamma_{14}q_5)}$. From \eqref{cc21}\eqref{cc22} we can further obtain $c_1=c_2$ and $d_1=d_2$. So $\vec\theta^{(1)}=\vec\theta^{(2)}$, which is a contradiction.

If there is only one of the $e_r$'s, which is $u_1$, w.l.o.g. let $e_1=u_1$ and $e_2=e_3=e_4=u_2$. We consider the moment where $a_2$ is ranked at the top and $a_1$ the second, which is $[\frac {e_1b_1} {1-b_1}, \frac {e_2b_2} {1-b_2}, \frac {e_3b_3} {1-b_3}, \frac {e_4b_4} {1-b_4}]$. Add $\vec\omega^{(1)}$ to it and we have $\vec\theta^{(be)}=[\frac {e_1} {1-b_1}, \frac {e_2} {1-b_2}, \frac {e_3} {1-b_3}, \frac {e_4} {1-b_4}]$, which is linearly dependent of $\mathbf{G}$. So there exists nonzero vector $\vec{\gamma_2}=[\gamma_{21},\gamma_{22},\gamma_{23}, \gamma_{24}]$ such that
\begin{equation}
\gamma_{21}+\gamma_{22}e_r+\gamma_{23}b_r+\gamma_{24}\frac {e_r} {1-b_r}=0\label{sst2}
\end{equation}
Let
\begin{align*}
f(x)&=\gamma_{21}+\gamma_{22}u_2+\gamma_{23}x+\gamma_{24}\frac {u_2} {1-x}\\
g(x)&=(1-x)f(x)=(1-x)(\gamma_{21}+\gamma_{22}u_2+\gamma_{23}x)+\gamma_{24}u_2
\end{align*}
If any coefficient of $g(x)$ is nonzero, then $g(x)$ has at most 2 different roots. As $g(x)=0$ holds for $b_2, b_3, b_4$, $\exists s\neq t$ s.t. $b_s=b_t$. Since $e_s=e_t=u_2$, from \eqref{cc21}\eqref{cc22} we know $c_s=c_t$ and $d_s=d_t$. So $\vec\theta^{(s)}=\vec\theta^{(t)}$. Otherwise we have $g(x)=f(x)=0$ for all $x$. So
\begin{equation*}
g(1)=\gamma_{24}u_2=0
\end{equation*}
Since $0<u_2<1$, we have $\gamma_{24}=0$. Substitute it into $f(x)$ we have $f(x)=\gamma_{21}+\gamma_{22}u_2+\gamma_{23}x=0$ holds for all $x$. So we have $\gamma_{21}+\gamma_{22}u_2=0$ and $\gamma_{23}=0$. Substitute $\gamma_{23}=\gamma_{24}=0$ into \eqref{sst2} we get $\gamma_{21}+\gamma_{22}e_r=0$, which holds for both $e_r=u_1$ and $e_r=u_2$. As $u_1\neq u_2$, we have $\gamma_{22}=0$. Then we have $\gamma_{21}=0$. This contradicts the nonzero requirement of $\vec{\gamma_2}$. So there exists $s\neq t$ s.t. $\vec\theta^{(s)}=\vec\theta^{(t)}$, which is a contradiction.

\end{proof}

Slightly abusing the notation, we say that a parameter of a $k$-PL is {\em identifiable}, if there does not exist a different parameter modulo label switching with the same probability distribution over the sample space. The next theorem proves that the Lebesgue measure (in the $km-1$ dimensional Euclidean space) of non-identifiable parameters of $k$-PLs over $m$ alternatives is $0$ (generic identifiability as is defined in Section \ref{sec:cont}). 
\begin{thm}\label{thm:gid}
For any $m\geq 6$ and any $1\leq k\leq\lfloor\frac {m-2} 2\rfloor!$, the $k$-PL over $m$ alternatives is generically identifiable.
\end{thm}
\begin{proof} Let us start with a high-level description of the proof. We  will focus on the parameters whose mixing coefficients (entries of $\vec\alpha$) are pairwise different. Formally, for any $r_1, r_2\in\{1, \ldots, k\}$,
\begin{equation}\label{eq:pairwisediff}
\alpha_{r_1}=\alpha_{r_2}\implies r_1=r_2
\end{equation}
Such parameters have Lebesgue measure $1$ because the parameters with any pair of identical mixing coefficients are in a lower dimensional space. 
The generic identifiability will be proved by analyzing the uniqueness of tensor decompositions. We will construct a rank-one tensor $\mathbf T^{(r)}(\vec\theta^{(r)})$ to represent the $r$th Plackett-Luce component. Then the $k$-PL can be represented by the weighted average of tensors of its components, i.e.~$\mathbf T(\vec\theta)=\sum^k_{r=1}\alpha_r\mathbf T^{(r)}(\vec\theta^{(r)})$. We now provide a set of two conditions for $\vec\theta=(\vec\alpha, \vec\theta^{(1)}, \ldots, \vec\theta^{(k)})$ to be identifiable, and then prove that both hold generically. 
\begin{enumerate}
\item[] {\bf Condition 1.} For every $\vec\gamma = (\vec\beta, \vec\gamma^{(1)}, \ldots, \vec\gamma^{(k)})$, where $\vec\beta\ne \vec\alpha$ modulo label switching (the set of entries in $\vec\beta$ is different from the set of entries in $\vec\alpha$),  we have $\mathbf T(\vec\gamma)\ne \mathbf T(\vec\theta)$.
\item[] {\bf Condition 2.} For every $\vec\gamma = (\vec\beta, \vec\gamma^{(1)}, \ldots, \vec\gamma^{(k)})$ that is different from $\vec \theta$ and $\vec\beta = \vec\alpha$, we have $\Pr_{k\text{-PL}}(\cdot|\vec\gamma)\ne \Pr_{k\text{-PL}}(\cdot|\vec\theta)$.
\end{enumerate}
It follows from the definition of identifiability that if both conditions hold, then $\vec\theta$ is identifiable. 


\noindent{\bf Condition 1 generically holds.} We first show that $\mathbf T(\vec\theta)$ has a unique decomposition generically, then prove that the uniqueness of decomposition implies Condition 1. 

To construct the rank-one tensor  $\mathbf T^{(r)}(\vec\theta^{(r)})$, we partition the set of alternatives into three subsets. In the rest of the proof we assume that $m$ is even. The theorem can be proved similarly for odd $m$'s.
\begin{align*}
S_A&=\{a_1, a_2, \cdots, a_{\frac {m-2} 2}\}\\
S_B&=\{a_{\frac {m} 2}, a_{\frac {m+2} 2}, \cdots, a_{m-2}\}\\
S_C&=\{a_{m-1}, a_{m}\}
\end{align*}
There are $n_1=n_2=\frac {m-2} 2!$ rankings over $S_A$ and $S_B$ respectively, and two rankings over $S_C$ ($n_3=2$). Let the three coordinates  in the tensor $\mathbf T^{(r)}(\vec\theta^{(r)})$ be $\mathbf p^{(r)}_A, \mathbf p^{(r)}_B, \mathbf p^{(r)}_C$  that represent probabilities of all rankings restricted to $S_A,S_B,S_C$  respectively. We note that $\mathbf p^{(r)}_A, \mathbf p^{(r)}_B, \mathbf p^{(r)}_C$ are functions of $\vec\theta^{(r)}$, which are often omitted. Let $\mathbf p(j)$ denote the $j$th entry of $\mathbf p$. Then for all $r=1, \ldots, k$, we have $\sum^{n_1}_{j=1}\mathbf p^{(r)}_A(j)=\sum^{n_2}_{j=1}\mathbf p^{(r)}_A(j)=\sum^{n_3}_{j=1}\mathbf p^{(r)}_A(j)=1$.

Then, for any rankings $V_A\in {\mathcal L}(S_A)$, $V_B\in {\mathcal L}(S_B)$, and $V_C\in {\mathcal L}(S_C)$, we can prove that $\Pr_\text{PL}(V_A,V_B, V_C|\vec\theta^{(r)})=\Pr_\text{PL}(V_A|\vec\theta^{(r)})\times \Pr_\text{PL}(V_B|\vec\theta^{(r)})\times \Pr_\text{PL}(V_C|\vec\theta^{(r)})$. That is, $V_A$, $V_B$ and $V_C$ are independent given $\vec\theta^{(r)}$. We will prove this result for a more general class of models called random utility models~\citep{Thurstone27:Law}, of which the Plackett-Luce model is a special case. 

\begin{claim}\label{claim:ind}
Given a random utility model (RUM) $\mathcal{M}(\vec\theta)$ over a set of $m$ alternatives $\mathcal A$, let $\mathcal A_1, \mathcal A_2$ be two non-overlapping subsets of $\mathcal A$, namely $\mathcal A_1, \mathcal A_2\subset\mathcal A$ and $\mathcal A_1\cap\mathcal A_2=\emptyset$. Let $V_1, V_2$ be rankings over $\mathcal A_1$ and $\mathcal A_2$, respectively, then we have
$\Pr(V_1, V_2|\vec\theta)=\Pr(V_1|\vec\theta)\Pr(V_2|\vec\theta).$
\end{claim}

\begin{proof} In an RUM, given a ground truth utility $\vec\theta=[\theta_1, \theta_2, \ldots, \theta_m]$ and a distribution $\mu_i(\cdot|\theta_i)$ for each alternative, an agent samples a random utility $X_i$ for each alternative independently with probability density function $\mu_i(\cdot|\theta_i)$. The probability of the ranking $a_{i_1}\succ a_{i_2}\succ\cdots\succ a_{i_m}$ is
\begin{align*}
&\Pr(a_{i_1}\succ\cdots\succ a_{i_m}|\vec\theta)=\Pr(X_{i_1}>X_{i_2}>\cdots>X_{i_m})\\
&=\int^\infty_{-\infty}\int^\infty_{x_{i_m}}\cdots\int^\infty_{x_{i_2}}\mu_{i_m}(x_{i_m})\mu_{i_{m-1}}(x_{i_{m-1}})\ldots\mu_{i_1}(x_{i_1})dx_{i_1}dx_{i_2}\ldots dx_{i_m}
\end{align*}

W.l.o.g.~we let $i_1=1,\ldots,i_m=m$. 
Let $\mathcal S_{X_{1}>X_{2}>\cdots>X_{m}}$ denote the subspace of ${\mathbb R}^m$ where $X_{1}>X_{2}>\cdots>X_{m}$  and let $\mu(\vec x|\vec\theta)$ denote $\mu_{m}(x_{m})\mu_{{m-1}}(x_{{m-1}})\ldots\\ \mu_{1}(x_{1})$. Thus we have
\begin{equation*}
\Pr(a_{1}\succ\cdots\succ a_{m}|\vec\theta)=\int_{\mathcal S_{X_{1}>X_{2}>\cdots>X_{m}}}\mu(\vec x|\vec\theta)d\vec x
\end{equation*}

We first prove the following claim.

\begin{claim}\label{claim:partial}
Given a random utility model $\mathcal M(\vec\theta)$, for any parameter $\vec\theta$ and any $\mathcal A_s\subseteq\mathcal A$, we let $\vec\theta_s$ denote the components of $\vec\theta$ for alternatives in $\mathcal A_s$, and let $V_s$ be a full ranking over $A_s$ (which is a partial ranking over $\mathcal A$). Then we have $
\Pr(V_s|\vec\theta)=\Pr(V_s|\vec\theta_s).$
\end{claim}

\begin{proof}
Let $m_s$ be the number of alternatives in $\mathcal A_s$. Let $\mathcal S_{X_{1}>X_{2}>\cdots>X_{m_s}}$ denote the subspace of ${\mathbb R}^{m_s}$ where $X_{1}>X_{2}>\cdots>X_{m_s}$. W.l.o.g.~let $V_s$ be
$a_{1}\succ a_{2} \cdots\succ a_{m_s}$. Then we have
\begin{align*}
&\Pr(V_s|\vec\theta)=\int_{\mathcal S_{X_{1}>X_{2}>\cdots>X_{m_s}}\times{\mathbb R}^{m-m_s}}\mu(\vec x|\vec\theta)d\vec x\\
&=\int^\infty_{-\infty}\int^\infty_{x_{m_s}}\cdots\int^\infty_{x_2}\int^\infty_{-\infty}\cdots\int^\infty_{-\infty}\mu_{m}(x_{m})\ldots\mu_{1}(x_{1})dx_{m_s+1}\cdots dx_{m}dx_{1}\ldots dx_{m_s}\\
&=\int^\infty_{-\infty}\int^\infty_{x_{m_s}}\cdots\int^\infty_{x_2}\mu_{m_s}(x_{m_s})\mu_{m_s-1}(x_{m_s-1})\ldots\mu_{1}(x_{1})dx_{1}dx_{2}\ldots dx_{m_s}\\
&=\int_{\mathcal S_{X_{1}>X_{2}>\cdots>X_{m_s}}}\mu(\vec x_s|\vec\theta_s)d\vec x=\Pr(V_s|\vec\theta_s)
\end{align*}
\end{proof}

Let $\mathcal A_1=\{a_{11}, a_{12}, \ldots, a_{1m_1}\}$ and $\mathcal A_2=\{a_{21}, a_{22}, \ldots, a_{2m_2}\}$. Without loss of generality we let $V_1$ and $V_2$ be $a_{11}\succ a_{12}\succ\cdots\succ a_{1m_1}$ and $a_{21}\succ a_{22}\succ\cdots\succ a_{2m_2}$ respectively. For any $\vec\theta$, let $\vec\theta_1$ denote the subvector of $\vec\theta$ on ${\mathcal A}_1$. Let ${\mathcal S}_1$ denote ${\mathcal S}_{X_{11}>X_{12}>\cdots>X_{1m_1}}$.  $\vec\theta_2$ and ${\mathcal S}_2$ are defined similarly. According to Claim~\ref{claim:partial}, we have
$
\Pr(V_1|\vec\theta)=\Pr(V_1|\vec\theta_1)=\int_{\mathcal S_1}\mu(\vec x_1|\vec\theta_1)d\vec x_1$ and $
\Pr(V_2|\vec\theta)=\Pr(V_2|\vec\theta_2)=\int_{\mathcal S_2}\mu(\vec x_2|\vec\theta_2)d\vec x_2
$. Then we have
\begin{align*}
\Pr(V_1,V_2|\vec\theta)&=\int_{\mathcal S_1\times \mathcal S_2\times{\mathbb R}^{m-m_1-m_2}}\mu(\vec x|\vec\theta)d\vec x\\
&=\int_{\mathcal S_1\times \mathcal S_2}\mu(\vec x_1,\vec x_2|\vec\theta_1,\vec\theta_2)d\vec x&\text{(Claim~\ref{claim:partial})}\\
&=\int_{\mathcal S_1}\int_{\mathcal S_2}\mu(\vec x_1|\vec\theta_1)\mu(\vec x_2|\vec\theta_2)d\vec x_1d\vec x_2&\text{(Fubini's Theorem)}\\
&=\int_{\mathcal S_1}\mu(\vec x_1|\vec\theta_1)d\vec x_1\int_{\mathcal S_2}\mu(\vec x_2|\vec\theta_2)d\vec x_2\\
&=\Pr(V_1|\vec\theta_1)\Pr(V_2|\vec\theta_2)
\end{align*}

\end{proof}

Because $S_A$, $S_B$, and $S_C$ are non-overlapping, by Claim~\ref{claim:ind}, it follows that
$$\mathbf T^{(r)}(\vec\theta^{(r)})={\mathbf p^{(r)}_A}\otimes{\mathbf p^{(r)}_B}\otimes{\mathbf p^{(r)}_C}.$$
We will prove the tensor $\mathbf T(\vec\theta)=\sum^k_{r=1}\alpha_r\mathbf T^{(r)}(\vec\theta^{(r)})$ has a unique decomposition by Kruskal's theorem \cite[Theorem~4a in][]{Kruskal77:Three}. To this end, we analyze the Kruskal's rank of three matrices, defined as follows.
$
\mathbf P_A = 
\begin{bmatrix}
\mathbf p^{(1)}_A & \ldots & \mathbf p^{(k)}_A
\end{bmatrix}
$,
$
\mathbf P_B= 
\begin{bmatrix}
\mathbf p^{(1)}_B & \ldots & \mathbf p^{(k)}_B
\end{bmatrix}
$
and
$
\mathbf P_C= 
\begin{bmatrix}
\mathbf p^{(1)}_C & \ldots & \mathbf p^{(k)}_C
\end{bmatrix}
$.
The Kruskal's rank of a matrix is the maximum number $K$ s.t. any $K$ columns of the matrix are linearly independent. Specifically, if a matrix has full column rank, then the Kruskal's rank of the matrix equals to the number of columns. Let $\krank(\mathbf A)$ denote the Kruskal's rank of matrix $\mathbf A$. We prove the following lemma, which states that when $k\leq\lfloor\frac {m-2} 2\rfloor!$, $\mathbf P_A$ and $\mathbf P_B$ generically have full (column) Kruskal's rank.

\begin{claim}\label{claim:genericfullrank}
Let $S$ be a set of $m_0$ alternatives. Define $\mathbf P_S = \begin{bmatrix}
\mathbf p^{(1)}_S & \ldots & \mathbf p^{(k)}_S\end{bmatrix}$, where each column $\mathbf p^{(r)}_S$ is an $m_0!$-dimensional column vector with probabilities of all full rankings over alternatives in $S$ given the $r$th Plackett-Luce component. If $k\le m_0!$, then $\krank(\mathbf P_S)=k$ generically holds.
\end{claim}

\begin{proof} It suffices to prove when $k=m_0!$, $\mathbf P_S$ has full rank generically. We will prove $\det(\mathbf P_S)\neq 0$ generically holds. Since $\det(\mathbf P_S)$ can be represented as a polynomial fraction with a nonzero denominator, we only need to prove there exist $k=m_0!$ Plackett-Luce components s.t. the numerator is nonzero ($\det(P_S)\neq 0$). This implies that the polynomial in the numerator is not identically $0$, and thus the Lebesgue measure of parameters resulting in a zero numerator is zero~\citep{Caron05:Zero}.

We prove by construction that there exist $\vec\theta^{(1)}, \ldots, \vec\theta^{(k)}$ s.t $\det(\mathbf P_S)\ne 0$. We will construct $k$ Plackett-Luce components s.t. $P_S$ is a diagonally dominant matrix, which implies $\det(\mathbf P_S)\ne 0$. To this end, we first show that given any ranking $V_0$, we can construct a parameter $\vec\theta$ s.t. $\Pr_{\vec\theta}(V_0)>1/2$. W.l.o.g. we define $V_0=a_1\succ a_2\succ\ldots\succ a_{m_0}$. Let $\sqrt[m_0-1]{1/2}<C<1$ be a constant. Then we let $\theta_1 = C, \theta_2 = C(1-\theta_1), \theta_3 = C(1-\theta_1-\theta_2), \ldots, \theta_{m_0}=1-\sum^{m_0-1}_{i=1}\theta_i$. It follows that $\Pr_{\vec\theta}(V_0) = C^{m_0-1}>1/2$. 

Recall that each column of $\mathbf P_S$ is one Plackett-Luce component and each row corresponds to a ranking over $S$. For each diagonal entry $(r, r)$ of $\mathbf P_S$, there is a corresponding ranking (denoted by $V_r$). Given the ranking $V_r$, we construct $\vec\theta^{(r)}$ using the method described above s.t. $\Pr_{\vec\theta^{(r)}}(V_r)>1/2$. Because entries of each column in $\mathbf P_S$ sum up to $1$, $\mathbf P^\top_S$ is a strictly diagonally dominant matrix. By Levy-Desplanques theorem, we have $\det(\mathbf P_S)\neq 0$.
\end{proof}

\cite{Kruskal77:Three} proved that if $\krank(\mathbf P_A)+\krank(\mathbf P_B)+\krank(\mathbf P_C)\ge 2k+2$, then $\mathbf T$ has a unique decomposition. By Claim~\ref{claim:genericfullrank}, when $k\leq\lfloor\frac {m-2} 2\rfloor!$, $\krank(\mathbf P_A)=\krank(\mathbf P_B)=k$ and $\krank(\mathbf P_C)=2$ generically hold. Thus, generically, tensor $\mathbf T(\vec\theta)$ has a unique decomposition.

We now prove that uniqueness of decomposition of $\mathbf T(\vec\theta)$ implies Condition 1, formally shown as the following claim.
\begin{claim}\label{claim:cond1}
For any $\vec\theta$ where $\mathbf T(\vec\theta)$ has a unique decomposition, Condition 1 holds. That is, for every $\vec\gamma = (\vec\beta, \vec\gamma^{(1)}, \ldots, \vec\gamma^{(k)})$, where $\vec\beta\ne \vec\alpha$ modulo label switching (the set of entries in $\vec\beta$ is different from the set of entries in $\vec\alpha$),  we have $\mathbf T(\vec\gamma)\ne \mathbf T(\vec\theta)$.
\end{claim}
\begin{proof}
Suppose for the purpose of contradiction, the decomposition of $\mathbf T(\vec\theta)$ is unique but Condition 1 does not hold. This means that there exists $\vec\gamma=(\vec\beta, \vec\gamma^{(1)}, \ldots,\\ \vec\gamma^{(k)})$ where $\vec\beta$ is not equal to $\vec\alpha$ modulo label switching, s.t. $\mathbf T(\vec\gamma)\ne \mathbf T(\vec\theta)$. Because components in $\vec\alpha$ are pairwise different, there exists $r_1\leq k$ such that $\alpha_{r_1}$is different from any element in $\vec\beta$, while $\mathbf T(\vec\theta)=\mathbf T(\vec\gamma)=\mathbf T$. We will show that for any $r_2=1, \ldots, k$, we have $\alpha_{r_1}\mathbf T^{(r_1)}(\vec\theta^{(r_1)})\neq\beta_{r_2}\mathbf T^{(r_2)}(\vec\gamma^{(r_2)})$, which contradicts the unique decomposition of $\mathbf T$. Recall that for any $r$, we have $\mathbf T^{(r)}(\cdot)={\mathbf p^{(r)}_A}\otimes{\mathbf p^{(r)}_B}\otimes{\mathbf p^{(r)}_C}$ and $\sum^{n_1}_{j=1}\mathbf p^{(r)}_A(j)=\sum^{n_2}_{j=1}\mathbf p^{(r)}_A(j)=\sum^{n_3}_{j=1}\mathbf p^{(r)}_A(j)=1$. The sum of all entries of $\mathbf T^{(r)}(\cdot)$ is one because
\begin{align*}
\sum^{n_1}_{j_1=1}\sum^{n_2}_{j_2=1}\sum^{n_3}_{j_3=1}\mathbf p^{(r)}_A(j_1)\mathbf p^{(r)}_B(j_2)\mathbf p^{(r)}_C(j_3)&=\sum^{n_1}_{j_1=1}\sum^{n_2}_{j_2=1}\mathbf p^{(r)}_A(j_1)\mathbf p^{(r)}_B(j_2)\sum^{n_3}_{j_3=1}\mathbf p^{(r)}_C(j_3)\\
&=\sum^{n_1}_{j_1=1}\mathbf p^{(r)}_A(j_1)\sum^{n_2}_{j_2=1}\mathbf p^{(r)}_B(j_2)=1.
\end{align*}

Thus, for any $r$, we have
\begin{equation}\label{eq:alphar}
\sum^{n_1}_{j_1=1}\sum^{n_2}_{j_2=1}\sum^{n_3}_{j_3=1}(\alpha_r\mathbf T^{(r)}(\cdot))_{j_1j_2j_3}=\alpha_r
\end{equation}

So the sums of all entries of $\alpha_{r_1}\mathbf T^{(r_1)}(\vec\theta^{(r_1)})$ and $\beta_{r_2}\mathbf T^{(r_2)}(\vec\gamma^{(r_2)})$ are $\alpha_{r_1}$ and $\beta_{r_2}$ respectively. We recall that for all $r_2$, we have $\alpha_{r_1}\neq\beta_{r_2}$. Therefore, $\alpha_{r_1}\mathbf T^{(r_1)}(\vec\theta^{(r_1)})\\ \neq\beta_{r_2}\mathbf T^{(r_2)}(\vec\gamma^{(r_2)})$, which is a contradiction. 
\end{proof}

This finishes the proof that {\bf Condition 1 generically holds}.

\noindent{\bf Condition 2 generically holds.} Unfortunately the tensor decomposition technique used for Condition 1 no longer works for Condition 2. This is because due to the partition of alternatives, the parameter of alternatives in one group can be arbitrarily scaled while the distribution of partial rankings restricted to the group does not change. For example, given $\vec\theta = (\vec\alpha, \vec\theta^{(1)}, \ldots, \vec\theta^{(k)})$, we can construct $\vec\gamma = (\vec\alpha, \vec\gamma^{(1)}, \ldots, \vec\gamma^{(k)})$ s.t. $\mathbf T(\vec\gamma)=\mathbf T(\vec\theta)$ in the following way. For any $r$, we let $\gamma^{(r)}_i=2\theta^{(r)}_i$ for $i=1, \ldots, m-2$ and $\gamma^{(r)}_i=\theta^{(r)}_i$ for $i=m-1, m$. Then we normalize each $\vec\gamma^{(r)}$ s.t. $\sum^m_{i=1}\gamma^{(r)}_i=1$. The resulting tensor $\mathbf T({\vec\gamma})=\mathbf T({\vec\theta})$ because the probability of rankings restricted to each group are exactly the same.

To address this problem we consider an additional tensor decomposition using a different partition of the alternatives, defined as follows: $
S'_A=\{a_3, a_5, \ldots, a_{m-1}\}, \\
S'_B=\{a_{2}, a_{4}, \ldots, a_{m-2}\},
S'_C=\{a_1, a_m\}$. Let $\mathbf T'(\vec\theta)$ denote the tensor under this partition. 
Similar to the proof for Condition 1, we can show that generically $\mathbf T'(\vec\theta)$ has a unique decomposition. Next, we will prove that for any $\vec\theta$ where $\mathbf T(\vec\theta)$ and $\mathbf T'(\vec\theta)$ have unique decompositions (which holds generically), Condition 2 holds.   

Suppose for the sake of contradiction that Condition 2 does not hold for $\vec\theta$ where $\mathbf T(\vec\theta)$ and $\mathbf T'(\vec\theta)$ have unique decompositions. Then, there exists $\vec\gamma=(\vec\beta, \vec\gamma^{(1)}, \ldots, \\ \vec\gamma^{(k)})$ that is different from $\vec\theta$ modulo label switching, such that $\Pr_{k\text{-PL}}(\cdot|\vec\gamma)=\Pr_{k\text{-PL}}(\cdot|\vec\theta)$. This means that $\mathbf T(\vec\gamma)=\mathbf T(\vec\theta)$ and  $\mathbf T'(\vec\gamma)=\mathbf T'(\vec\theta)$.

We first match the components in $\vec\gamma$ and $\vec\theta$. Recall from Claim~\ref{claim:cond1} that uniqueness of $\mathbf T(\vec\theta)$ implies Condition 1. Because both $\mathbf T(\vec\theta)$ and $\mathbf T'(\vec\theta)$ have unique decompositions, Condition 1 must hold for both. It follows that the entries of $\vec\beta$ are exactly entries of $\vec\alpha$, otherwise $\mathbf T(\vec\theta)\ne \mathbf T(\vec\gamma)$, which is a contradiction. Since entries of $\vec\alpha$ are pairwise different, there is a unique way of matching components in $\vec\theta$ to components in $\vec\gamma$ by matching the corresponding mixing coefficients. Consequently, we can relabel components in $\vec\gamma$ s.t.~the $\vec\beta$ becomes exactly $\vec\alpha$. Let $\vec\gamma'=(\vec\alpha, \vec\gamma'^{(1)}, \ldots, \vec\gamma'^{(k)})$ denote the $\vec\gamma$ parameter after relabeling the components. Thus we have $\mathbf T(\vec\gamma')=\mathbf T(\vec\gamma)=\mathbf T(\vec\theta)$ and $\mathbf T'(\vec\gamma')=\mathbf T'(\vec\gamma)=\mathbf T'(\vec\theta)$. Because $\vec\gamma\ne \vec\theta$ modulo label switching, we have $\vec\gamma'\ne\vec\theta$, which implies that there exists $r^*$ s.t.
\begin{equation}\label{eq:neq}
\vec\gamma'^{(r^*)}\neq\vec\theta^{(r^*)}.
\end{equation}

Next, we will show that from the uniqueness of decomposition of $\mathbf T(\vec\theta)$ and $\mathbf T'(\vec\theta)$, we must have $\vec\gamma'^{(r^*)}=\vec\theta^{(r^*)}$, which is a contradiction. To this end, we show how to uniquely determine the parameters of $k$-PL (i.e.~$\vec\gamma'^{(r^*)}$ and $\vec\theta^{(r^*)}$) from decompositions of $\mathbf T(\vec\theta)$ and $\mathbf T'(\vec\theta)$.

$\mathbf T(\vec\gamma')=\mathbf T(\vec\theta)$ has a unique decomposition means that for any $r_1\in\{1, \ldots, k\}$, there exists $r_2\in\{1, \ldots, k\}$ s.t. $\alpha_{r_1}\mathbf T^{(r_1)}(\vec\gamma'^{(r_1)})=\alpha_{r_2}\mathbf T^{(r_2)}(\vec\theta^{(r_2)})$, which implies 
$$\sum^{n_1}_{j_1=1}\sum^{n_2}_{j_2=1}\sum^{n_3}_{j_3=1}(\alpha_{r_1}\mathbf T^{(r_1)}(\vec\gamma'^{(r_1)}))_{j_1j_2j_3}=\sum^{n_1}_{j_1=1}\sum^{n_2}_{j_2=1}\sum^{n_3}_{j_3=1}(\alpha_{r_2}\mathbf T^{(r_2)}(\vec\theta^{(r_2)}))_{j_1j_2j_3}$$

By \eqref{eq:alphar}, the above equation implies $\alpha_{r_1}=\alpha_{r_2}$. It follows from \eqref{eq:pairwisediff} that $r_1=r_2$. Namely, we have $\alpha_r\mathbf T^{(r)}(\vec\gamma'^{(r)})=\alpha_r\mathbf T^{(r)}(\vec\theta^{(r)})$ for all $r=1, \ldots, k$. It follows that for all $r=1, \ldots, k$, we have $\mathbf T^{(r)}(\vec\gamma'^{(r)})=\mathbf T^{(r)}(\vec\theta^{(r)})$, which means $\mathbf T^{(r^*)}(\vec\gamma'^{(r^*)})=\mathbf T^{(r^*)}(\theta^{(r^*)})$. Similarly, we also have $\mathbf T'^{(r^*)}(\vec\gamma'^{(r^*)})=\mathbf T'^{(r^*)}(\theta^{(r^*)})$. 

Recall that $n_1, n_2$ and $n_3$ are numbers of rankings over alternatives in  $S_A$, $S_B$, and $S_C$ respectively and for any $r=1, \ldots, k$, $\mathbf T^{(r)}(\vec\theta)={\mathbf p^{(r)}_A}\otimes{\mathbf p^{(r)}_B}\otimes{\mathbf p^{(r)}_C}$ is a rank-one tensor where $\sum^{n_1}_{j=1}\mathbf p^{(r)}_A(j)=\sum^{n_2}_{j=1}\mathbf p^{(r)}_B(j)=\sum^{n_3}_{j=1}\mathbf p^{(r)}_C(j)=1$. Given $\mathbf T^{(r)}(\vec\theta)$, ${\mathbf p^{(r)}_A}$, ${\mathbf p^{(r)}_B}$ and ${\mathbf p^{(r)}_C}$ can be easily obtained by normalizing the corresponding entries of $\mathbf T^{(r)}(\vec\theta)$. E.g., ${\mathbf p^{(r)}_A}$ can be obtained by normalizing all $(\cdot, 1, 1)$ entries. We claim that $\mathbf p^{(r)}_A$ is uniquely determined by $\mathbf T^{(r)}(\vec\theta^{(r)})$ because (i) $\mathbf p^{(r)}_A$ is proportional to $(\cdot, 1, 1)$ entries because $\mathbf T^{(r)}(\vec\theta)={\mathbf p^{(r)}_A}\otimes{\mathbf p^{(r)}_B}\otimes{\mathbf p^{(r)}_C}$; and (ii) $\sum^{n_1}_{j=1}\mathbf p^{(r)}_A(j)=1$. Specifically $\mathbf p^{(r^*)}_A$ is uniquely determined by $\mathbf T^{(r^*)}(\vec\theta^{(r^*)})$.
Similarly $\mathbf p^{(r^*)}_B$ and $\mathbf p^{(r^*)}_C$ are uniquely determined by $\mathbf T^{(r^*)}(\vec\theta^{(r^*)})$ and $\mathbf p'^{(r^*)}_A, \mathbf p'^{(r^*)}_B$ and $\mathbf p'^{(r^*)}_C$ are uniquely determined by $\mathbf T'^{(r^*)}(\vec\theta^{(r^*)})$.

Next, we will prove $\mathbf p^{(r^*)}_A, \mathbf p^{(r^*)}_B, \mathbf p^{(r^*)}_C, \mathbf p'^{(r^*)}_A, \mathbf p'^{(r^*)}_B$ and $\mathbf p'^{(r^*)}_C$ uniquely determines $\vec\theta^{(r^*)}$, which contradicts \eqref{eq:neq}. Now we focus on $\vec\gamma'^{(r^*)}$ and $\vec\theta^{(r^*)}$ restricted to $S_A$. We claim that there exists a constant $C_A$ s.t. for all $i = 1, 2, \ldots, \frac {m-2} 2$ (i.e. $a_i\in S_A$), $\gamma'^{(r^*)}_i=C_A\theta^{(r^*)}_i$. The reason is as follows.

For all $i=1,\ldots, \frac {m-2} 2$, we consider the probability of the event that $a_i$ is ranked at the top among the alternatives in $S_A$ given $\vec\theta^{(r^*)}$ and $\vec\gamma'^{(r^*)}$. By Claim~\ref{claim:partial}, the two probabilities must be the same. Therefore for all $i=1,\ldots, \frac {m-2} 2$, we have 
\begin{align*}
\Pr(a_i\text{ top among } S_A|\vec\theta^{(r^*)})&=\frac {\theta^{(r^*)}_i} {\sum^{\frac {m-2} 2}_{i=1}\theta^{(r^*)}_i}\\
\Pr(a_i\text{ top among } S_A|\vec\gamma'^{(r*)})&=\frac {\gamma'^{(r^*)}_i} {\sum^{\frac {m-2} 2}_{i=1}\gamma'^{(r^*)}_i}
\end{align*}
This means there exists a constant $C_A=\frac {\sum^{\frac {m-2} 2}_{i=1}\gamma'^{(r^*)}_i} {\sum^{\frac {m-2} 2}_{i=1}\theta^{(r^*)}_i}$ s.t. 
\begin{equation}\label{t1}
\gamma'^{(r^*)}_i=C_A\theta^{(r^*)}_i, \text{ for } i=1, \ldots, \frac {m-2} 2
\end{equation}
Similarly there exist $C_B, C_C$, s.t.
\begin{align}
\gamma'^{(r^*)}_i &= C_B\theta^{(r^*)}_i, \text{ for } i=\frac m 2, \ldots, m-2\label{t2}\\
\gamma'^{(r^*)}_i &= C_C\theta^{(r^*)}_i, \text{ for } i= m-1, m\label{t3}
\end{align}
Similarly from $\mathbf T'^{(r)}(\vec\gamma^{(r)})=\mathbf T'^{(r)}(\vec\theta^{(r)})$, there exist constants $C'_A, C'_B$ and $C'_C$ s.t. 
\begin{align}
\gamma'^{(r^*)}_i &= C'_A\theta^{(r^*)}_i, \text{ for } i=3, 5, \ldots, m-1\label{tp1}\\
\gamma'^{(r^*)}_i &= C'_B\theta^{(r^*)}_i, \text{ for } i=2, 4, \ldots, m-2\label{tp2}\\
\gamma'^{(r^*)}_i &= C'_C\theta^{(r^*)}_i, \text{ for } i= 1, m\label{tp3}
\end{align}
Therefore, we have $C_B=C'_B$ by Equations~\eqref{t2} and \eqref{tp2} (let $i=m-2$); $C'_B=C_A$ by Equations~\eqref{tp2} and \eqref{t1} (let $i=2$); $C_A=C'_C$ by Equations~\eqref{t1} and \eqref{tp3} (let $i=1$); $C'_C=C_C$ by Equations~\eqref{tp3} and \eqref{t3} (let $i=m$); $C_C=C'_A$ by Equations~\eqref{t3} and \eqref{tp1} (let $i=m-1$). So we let $C_B=C'_B=C_A=C'_C=C_C=C'_A=C$. Then for all $i=1, \ldots, m$, we have $\gamma'^{(r^*)}_i=C\theta^{(r^*)}_i$. Because $\sum^m_{i=1}\gamma'^{(r^*)}_i=\sum^m_{i=1}\theta^{(r^*)}_i=1$, we have $C=1$ and $\vec\gamma'^{(r^*)}=\vec\theta^{(r^*)}$, which is a contradiction.

As we have proved, the Lebesgue measure of parameters where tensor $\mathbf T(\vec\theta)$ (or $\mathbf T'(\vec\theta)$) has a nonunique decomposition is $0$. Therefore, the Lebesgue measure of parameters $\vec\theta$ where both $\mathbf T(\vec\theta)$ and $\mathbf T'(\vec\theta)$ have unique decompositions is also $0$. This finishes the proof.
\end{proof}

\begin{figure*}[htp]
\centerline{\includegraphics[width=\textwidth]{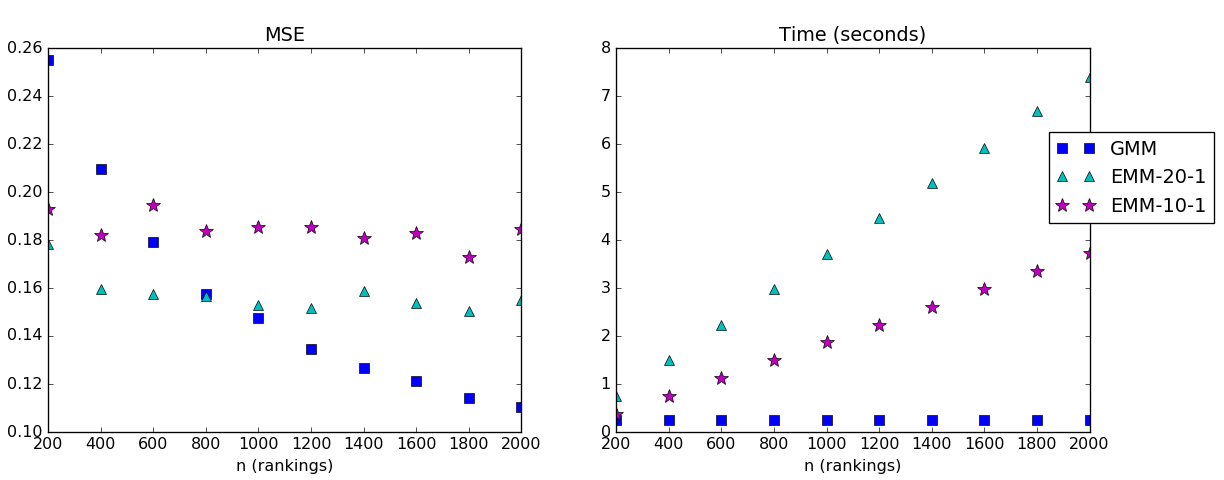}}
\caption{The MSE and running time of GMM and EMM.  EMM-20-1 is 20 iterations of EMM overall with 1 iteration of MM for each M step.  Likewise, EMM-10-1 is 10 iterations of EMM with 1 iteration of MM each time.  Values are calculated over 2000 datasets.}
\label{fig:gmmemm}
\end{figure*}

\section{A Generalized Method of Moments Algorithm for $2$-PL}

In a {\em generalized method of moments} (GMM) algorithm, a set of $q\ge 1$ {\em moment conditions} $g(V,\vec\theta)$ are specified. Moment conditions are functions of the parameter and the data, whose expectations are zero at the ground truth. $g(V,\vec\theta)\in{\mathbb R}^q$ has two inputs: a data point $V$ and a parameter $\vec\theta$. For any $\vec\theta^*$, the expectation of any moment condition should be zero at $\vec\theta^*$, when the data are generated from the model given $\vec\theta^*$. Formally $E[g(V, \vec\theta^\ast)]=\vec 0$. In practice the observed moment values should match the theoretical values from the model. In our algorithm, each moment condition corresponds to an event in the data, e.g. $a_1$ is ranked at the top. We use {\em moments} to denote such events. Given any preference profile $P$, we let $g(P,\vec\theta)=\frac 1 n\sum_{V\in P}g(V,\vec\theta)$, which is a function of $\vec\theta$. 
The GMM algorithm we will use then computes the parameter that minimizes the $2$-norm of the empirical moment conditions in the following way.
\begin{equation}\label{equ:gmm}
\text{GMM}_g(P)=\inf_{\vec{\theta}}||g(P,\vec{\theta})||_{2}^2
\end{equation}
In this paper, we will show results for $m=4$ and $k=2$. Our GMM works for other combinations of $k$ and $m$, if the model is identifiable. Otherwise the estimator is not consistent. For $m=4$ and $k=2$, the parameter of the $2$-PL is $\vec\theta=(\alpha,\vec\theta^{(1)},\vec\theta^{(2)})$. We will use the following $q=20$ moments from three categories.

(i) There are four moments, one for each of the four alternatives to be ranked at the top. Let $\{g_{i}:i\le 4\}$ denote the four moment conditions. Let $p_{i}=\alpha\theta^{(1)}_{i}+(1-\alpha)\theta^{(2)}_{i}$. For any $V\in{\mathcal L}({\mathcal A})$, we have $g_{i}(V,\vec\theta)=1-p_{i}$ if and only if $a_i$ is ranked at the top of $V$; otherwise $g_{i}(V,\vec\theta)=-p_{i}$.

(ii) There are $12$ moments, one for each combination of top-$2$ alternatives in a ranking. Let $\{g_{i_1i_2}:i_1\neq i_2\le 4\}$ denote the $12$ moment conditions. Let $p_{i_1i_2}=\alpha\frac {\theta^{(1)}_{i_1}\theta^{(1)}_{i_2}} {1-\theta^{(1)}_{i_1}}+(1-\alpha)\frac {\theta^{(2)}_{i_1}\theta^{(2)}_{i_2}} {1-\theta^{(2)}_{i_1}}$. For any $V\in{\mathcal L}({\mathcal A})$, we have $g_{i_1i_2}(V,\vec\theta)=1-p_{i_1i_2}$ if and only if $a_{i_1}$ is ranked at the top and $a_{i_2}$ is ranked at the second  in $V$; otherwise $g_{i_1i_2}(V,\vec\theta)=-p_{i_1i_2}$.

(iii) There are four moments that correspond to the following four rankings $a_1\succ a_2\succ a_3\succ a_4$, $a_2\succ a_3\succ a_4\succ a_1$, $a_3\succ a_4\succ a_1\succ a_2$, $a_4\succ a_1\succ a_2\succ a_3$. The corresponding $g_{i_1i_2i_3i_4}$'s are defined similarly.

To illustrate how GMM works, we give the following example. Suppose the data $P$ contain the following  20 rankings.

$\hfill
\begin{array}{c@{\ @\ }c}
8 & a_1\succ a_2\succ a_3\succ a_4\\
4 & a_2\succ a_3\succ a_4\succ a_1\\
3 & a_3\succ a_4\succ a_1\succ a_2\\
3 & a_4\succ a_1\succ a_2\succ a_3\\
2 & a_3\succ a_1\succ a_4\succ a_2\
\end{array}
\hfill$

Then, for example, $g_1(P, \vec\theta)=\frac 8 {20}-(\alpha\theta^{(1)}_1+(1-\alpha)\theta^{(2)}_1)$, corresponding to the moment that $a_1$ is ranked at top (category i). $g_{34}(P, \vec\theta)=\frac 3 {20}-(\frac {\alpha\theta^{(1)}_3\theta^{(1)}_4} {\theta^{(1)}_1+\theta^{(1)}_2+\theta^{(1)}_4}+\frac {(1-\alpha)\theta^{(2)}_3\theta^{(2)}_4} {\theta^{(2)}_1+\theta^{(2)}_2+\theta^{(2)}_4})$, corresponding to the moment of $a_3\succ a_4\succ$ others (category ii). 
$g_{2341}(P,\vec\theta)=\frac 4 {20}-(\frac {\alpha\theta^{(1)}_2\theta^{(1)}_3\theta^{(1)}_4} {(\theta^{(1)}_3+\theta^{(1)}_4+\theta^{(1)}_1)(\theta^{(1)}_4+\theta^{(1)}_1)}+\frac {(1-\alpha)\theta^{(2)}_2\theta^{(2)}_3\theta^{(2)}_4} {(\theta^{(2)}_3+\theta^{(2)}_4+\theta^{(2)}_1)(\theta^{(2)}_4+\theta^{(2)}_1)})$, corresponding to the moment of $a_2\succ a_3\succ a_4\succ a_1$ (category iii). Remember that $\sum^4_{i=1}\theta^{(1)}_i=\sum^4_{i=1}\theta^{(2)}_i=1$.

The choices of these moment conditions are based on the proof of Theorem~\ref{thm:id}, so that the $2$-PL is strictly identifiable w.r.t.~these moment conditions. Therefore, our simple GMM algorithm is the following.

\begin{algorithm}%
{\bf Input}: Preference profile $P$ with $n$ full rankings.

Compute the frequency of each of the $20$ moments\\
Compute the output according to (\ref{equ:gmm})

\caption{GMM for $2$-PL}
\label{GMMalg}
\end{algorithm}

The theoretical guarantee of our GMM is its consistency, as we defined in Section \ref{sec:cont}.
\begin{thm}\label{thm:gmm}
Algorithm~\ref{GMMalg} is consistent w.r.t.~$2$-PL, where there exists $\epsilon>0$ such that each parameter is in $[\epsilon,1]$.
\end{thm}

\begin{proof}
We will check all assumptions in Theorem 3.1 in \cite{Hall05:Generalized}. 

Assumption 3.1: Strict Stationarity: the ($n\times 1$) random vectors $\{v_t; -\infty<t<\infty\}$ form a strictly stationary process with sample space $\mathcal S\subseteq\mathbb{R}^n$. 

As the data are generated i.i.d., the process is strict stationary.

Assumption 3.2: Regularity Conditions for $g(\cdot,\cdot)$: the function $g: \mathcal S\times\Theta\rightarrow\mathbb{R}^q$ where $q<\infty$, satisfies: (i) it is continuous on $\Theta$ for each $P\in\mathcal S$; (ii) $E[g(P,\vec\theta)]$ exists and is finite for every $\theta\in\Theta$; (iii) $E[g(P,\vec\theta)]$ is continuous on $\Theta$.

Our moment conditions satisfy all the regularity conditions since $g(P,\vec\theta)$ is continuous on $\Theta$ and bounded in $[-1,1]^9$.

Assumption 3.3: Population Moment Condition. The random vector $v_t$ and the parameter vector $\theta_0$ satisfy the ($q\times 1$) population moment condition: $E[g(P, \theta_0)]=0$.

This assumption holds by the definition of our GMM. 

Assumption 3.4 Global Identification. $E[g(P, \vec{\theta'})]\neq 0$ for all $\vec{\theta'}\in\Theta$ such that $\vec{\theta'}\neq\theta_0$.

This is proved in Theorem~\ref{thm:id}.

Assumption 3.7 Properties of the Weighting Matrix. $W_t$ is a positive semi-definite matrix which converges in probability to the positive definite matrix of constants $W$.

This holds because $W=I$.

Assumption 3.8 Ergodicity. The random process $\{v_t;-\infty<t<\infty\}$ is ergodic.

Since the data are generated i.i.d., the process is ergodic.

Assumption 3.9 Compactness of $\Theta$. $\Theta$ is a compact set.

$\Theta=[\epsilon, 1]^{9}$ is compact.

Assumption 3.10 Domination of $g(P,\vec\theta)$. $E[\sup_{\theta\in\Theta}||g(P,\vec\theta)||]<\infty$.

This assumption holds because all moment conditions are finite.

Theorem 3.1 Consistency of the Parameter Estimator.
If Assumptions 3.1-3.4 and 3.7-3.10 hold then $\hat\theta_T\overset{p}\to\theta_0$

\end{proof}

Originally all parameters lie in open intervals (0, 1]. The $\epsilon$ requirement in the theorem is introduced to make the parameter space compact, i.e. all parameters are chosen from closed intervals. The proof is done by applying Theorem 3.1 in~\cite{Hall05:Generalized}. The main hardness is the identifiability of $2$-PL w.r.t.~the moment conditions used in our GMM. Our proof of the identifiability of $2$-PL (Theorem~\ref{thm:id}) only uses the $20$ moment conditions described above.\footnote{In fact our proof only uses $16$ of them ($4$ out of the $12$ moment conditions in category (ii) are redundant). However, our synthetic experiments show that using $20$ moments improves statistical efficiency without sacrificing too much computational efficiency.}

\noindent {\bf Complexity of GMM.} For learning $k$-PL with $m$ alternatives and $n$ rankings with EMM, each E-step performs $O(nk^2)$ operations and each iteration of the MM algorithm for the M-step performs $O(m^2nk)$ operations. Our GMM for $k=2$ and $m=4$ has overall complexity $O(n)$. The complexity of calculating moments is $O(n)$ and the complexity of optimization depends only on $m$ and $k$. 

\begin{figure*}[htp]
\centerline{\includegraphics[width=\textwidth]{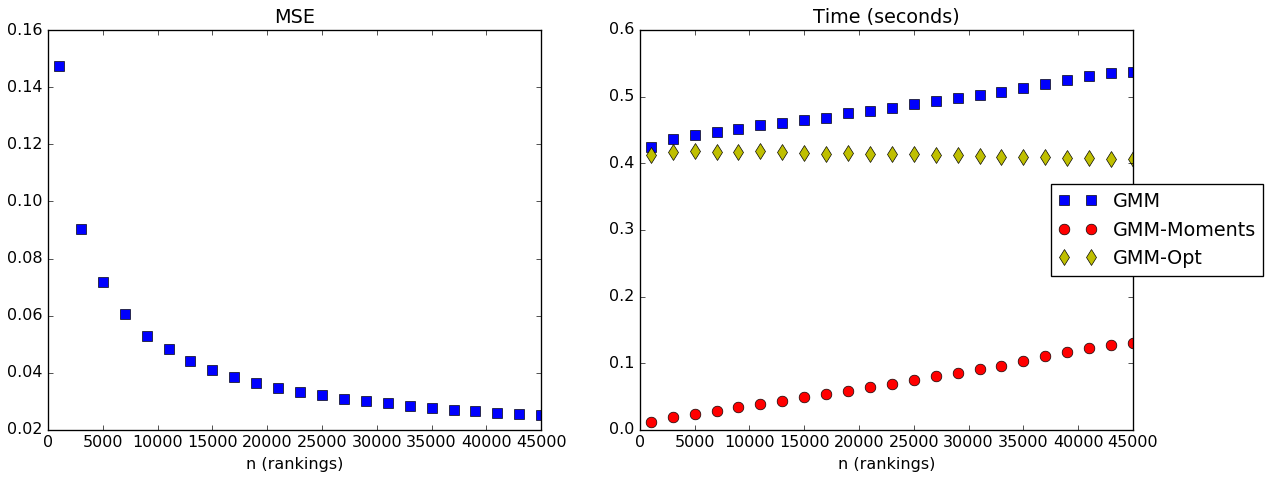}}
\caption{The MSE and running time of GMM.  GMM-Moments is the time to calculate the moment condition values observed in the data and GMM-Opt is the time to perform the optimization.  Values are calculated over 50000 trials.}
\label{fig:gmm}
\end{figure*}

\section{Experiments}
\label{sec:exp}
The performance of our GMM algorithm (Algorithm~\ref{GMMalg}) is compared to the EMM algorithm~\cite{Gormley08:Exploring} for $2$-PL with respect to running time and statistical efficiency for synthetic data.  The synthetic datasets are generated as follows.

$\bullet$ Generating the ground truth: for \(k=2\) mixtures and \(m=4\) alternatives, the mixing coefficient \(\alpha{^\ast}\) is generated uniformly at random and the Plackett-Luce components $\vec\theta^{(1)}$ and $\vec\theta^{(2)}$  are each generated from the Dirichlet distribution Dir\((\vec{1})\).

$\bullet$ Generating data: given a ground truth \(\vec{\theta}^\ast\), we generate each ranking with probability \(\alpha{^\ast}\) from the PL model parameterized by $\vec\theta^{(1)}$ and with probability \(1-\alpha{^\ast}\) from the PL model parameterized by $\vec\theta^{(2)}\) up to 45000 full rankings.

The GMM algorithm is implemented in Python 3.4 and termination criteria for the optimization are convergence of the solution and the objective function values to within \(10^{-10}\) and \(10^{-6}\) respectively. The optimization of (\ref{equ:gmm}) uses  the {\sf fmincon} function through the MATLAB Engine for Python.

The EMM algorithm is also implemented in Python 3.4 and the E and M steps are repeated together for a fixed number of iterations without convergence criteria.  %
The running time of EMM is largely determined by the total number of iterations of the MM algorithm and we look to compare the best EMM configuration with comparable running time to GMM.   We have tested all configurations of EMM with 10 and 20 overall MM iterations, respectively. We found that the optimal configurations are EMM-10-1 and EMM-20-1 (shown in Figure~\ref{fig:gmmemm}, results for other configurations are omitted), where EMM-20-1 means $20$ iterations of E step, each of which uses $1$ MM iteration.

We use the Mean Squared Error (MSE) as the measure of statistical efficiency defined as MSE~$=E(\parallel \vec{\theta} - \vec{\theta}^\ast \parallel_{2}^{2})$.

All experiments are run on an Ubuntu Linux server with Intel Xeon E5 v3 CPUs each clocked at 3.50 GHz.

\noindent{\bf Results.} The comparison of the performance of the GMM algorithm to the EMM algorithm is presented in Figure \ref{fig:gmmemm} for up to \(n = 2000\) rankings.  Statistics are calculated over 2000 trials (datasets).  We observe that:

$\bullet$ 
GMM is significantly faster than EMM.  The running time of GMM grows at a much slower rate in $n$.  %

$\bullet$ 
GMM achieves competitive MSE and the MSE of EMM does not improve over $n$. In particular, when $n$ is more than  $1000$, GMM  achieves smaller MSE.

The implication is that GMM may be better suited for reasonably large datasets where running time becomes infeasibly large with EMM. Moreover, it is possible that the GMM algorithm can be further improved by using a more accurate optimizer or another set of moment conditions.  GMM can also be used to provide a good initial point for other methods such as the EMM.

For larger datasets, the performance of the GMM algorithm is shown in Figure \ref{fig:gmm} for up to \(n = 45000\) rankings calculated over 50000 trials.
As the data size increases, GMM converges toward the ground truth, which verifies our consistency theorem (Theorem~\ref{thm:gmm}).  The overall running time of GMM shown in the figure is comprised of the time to calculate the moments from data (GMM-Moments) and the time to optimize the objective function (GMM-Opt).  The time for calculating the moment values increases linearly in \(n\), but it is dominated by the time to perform the optimization.

\section{Summary and Future Work}
In this paper we address the problem of identifiability and efficient learning for Plackett-Luce mixture models. We show that for any $k\geq 2$, $k$-PL for no more than $2k-1$ alternatives is non-identifiable and this bound is tight for $k=2$.
For generic identifiability, we prove that the mixture of $k$ Plackett-Luce models over $m$ alternatives is {\em generically identifiable} if $k\leq\lfloor\frac {m-2} 2\rfloor!$. We also propose a GMM algorithm for learning $2$-PL with four or more alternatives. Our experiments show that our GMM algorithm is significantly faster than the EMM algorithm in \cite{Gormley08:Exploring}, while achieving competitive statistical efficiency.

There are many directions for future research. An obvious open question is whether $k$-PL is identifiable for $2k$ alternatives for $k\ge 3$, which we conjecture to be true. It is  important to study how to efficiently check whether a learned parameter is identifiable for $k$-PL when $m<2k$. Can we further improve the statistical efficiency and computational efficiency for learning $k$-PL? We also plan to develop efficient implementations of  our GMM algorithm and apply it widely to various learning problems with big rank data.

\pdfbookmark[1]{Acknowledgements}{bkmkacks}
\section*{Acknowledgements}

This work is supported by the National Science Foundation under grant IIS-1453542 and a Simons-Berkeley research fellowship. We thank all reviewers for helpful comments and suggestions.

\newpage
\pdfbookmark[1]{References}{bkmkrefs}
\bibliography{references}
\bibliographystyle{plainnat}

\end{document}